\PassOptionsToPackage{dvipsnames, table, xcdraw}{xcolor}
\documentclass[]{fairmeta}
\usepackage{url}

\usepackage{amsthm}

\usepackage{minitoc}

\usepackage{array}
\definecolor{linkColor}{HTML}{E74C3C}
\definecolor{pearcomp}{HTML}{B97E29}
\definecolor{citeColor}{HTML}{2980B9}
\definecolor{urlColor}{HTML}{1D2DEC}
\definecolor{conjColor}{HTML}{9ab569}
\usepackage{pgfplots}
\pgfplotsset{compat=1.18}
\usepackage{amsmath,amssymb}
\usepackage[capitalise]{cleveref}
\usepackage{mathtools}
\usepackage{stmaryrd}
\pgfmathdeclarefunction{gauss}{3}{
  \pgfmathparse{#3/(#2*sqrt(2*pi))*exp(-((x-#1)^2)/(2*#2^2))}
}
\pgfmathdeclarefunction{mingauss}{4}{
  \pgfmathparse{#4/(#3*sqrt(2*pi))*exp(-max((x-#1)^2, (x-#2)^2)/(2*#3^2))}
}

\tikzset{
  invisible/.style={opacity=0},
  visible on/.style={alt={#1{}{invisible}}},
  alt/.code args={<#1>#2#3}{
    \alt<#1>{\pgfkeysalso{#2}}{\pgfkeysalso{#3}}
  },
}

\newtheorem{definition}{\textbf{Definition}}

\newtheorem{lemma}{\textbf{Lemma}}
\newtheorem{theorem}{\textbf{Theorem}}

\newtheorem{proposition}{\textbf{Proposition}}

\usepackage{bbold}

\usepackage[normalem]{ulem}
\usetikzlibrary{shapes, positioning}

\renewcommand{\cite}[1]{\citep{#1}}

\usepackage{color}
\definecolor{cm}{RGB}{0,0,200}
\definecolor{purple}{RGB}{200,0,200}

\usepackage[intoc]{nomencl}
\makenomenclature

\makeatletter
\newcommand{\vast}{\bBigg@{2.5}}
\newcommand{\Vast}{\bBigg@{5}}
\makeatother

\DeclareMathOperator{\E}{\mathbb{E}}

\newcommand{\Cov}{\mathbb{C}\mathrm{ov}}

\usepackage{bbm}

\usepackage[utf8]{inputenc}
\usepackage[LGR,T1]{fontenc}

\usepackage{breqn}
\usepackage{enumitem,kantlipsum}

\newcommand{\KL}{\mathrm{KL}}

\definecolor{rliableolive}{HTML}{BBCC33}
\definecolor{rliableblue}{HTML}{77AADD}
\definecolor{rliablered}{HTML}{EE8866}

\definecolor{takeawaycolor}{HTML}{FFF1E6}
\definecolor{takeawaycolor2}{HTML}{F2F6FF}
\definecolor{takeawaycolor3}{HTML}{EAF6EF}
\definecolor{takeawaycolor4}{HTML}{E6F4F1}
\definecolor{takeawaycolor5}{HTML}{FFF9E5}
\definecolor{takeawayborder}{HTML}{FFD8C2}
\definecolor{takeawayheader}{HTML}{0B3C8C}

\tcbset{
  aibox/.style={
    width=\linewidth,
    top=9pt,
    bottom=4pt,
    colback=takeawaycolor5!100!white,
    colframe=black!75!white,
    colbacktitle=black!75!white,
    boxrule=0.75pt,
    enhanced,
    center,
    attach boxed title to top left={yshift=-0.1in,xshift=0.15in},
    boxed title style={boxrule=0pt,colframe=white,},
  }
}
\newtcolorbox{AIbox}[2][]{aibox,title=#2,#1}

\tcbset{
  greenaibox/.style={
    width=\linewidth,
    top=9pt,
    bottom=4pt,
    colback=takeawaycolor3!100!white,
    colframe=black!75!white,
    colbacktitle=black!75!white,
    boxrule=0.75pt,
    enhanced,
    center,
    attach boxed title to top left={yshift=-0.1in,xshift=0.15in},
    boxed title style={boxrule=0pt,colframe=white,},
  }
}
\newtcolorbox{greenAIbox}[2][]{greenaibox,title=#2,#1}

\usepackage{algorithm}
\usepackage{algorithmicx}
\usepackage[noend]{algpseudocode}
\usepackage{enumitem,kantlipsum}

\usepackage{booktabs}
\usepackage{wrapfig}

\tcbuselibrary{skins}
\usepackage{tabularx}

\usepackage{boxedminipage}
\usepackage{lmodern}
\usepackage{makecell}

\makeatletter
\newcommand{\appendixtoc}{%
  \section*{Appendices}%
  \@starttoc{app}%
}

\newcommand{\appsection}[1]{%
  \section{#1}%
  \addcontentsline{app}{section}{\protect\numberline{\thesection}#1}%
}

\newcommand{\appsubsection}[1]{%
  \subsection{#1}%
  \addcontentsline{app}{subsection}{\protect\numberline{\thesubsection}#1}%
}

\makeatother

\title{\fontsize{20pt}{24pt}\selectfont Reinforcement Learning from Rich Feedback with Distributional DAgger}

\author{Rishabh Agrawal} \author{Jacob Fein-Ashley} \author{Paria Rashidinejad}

\affiliation{University of Southern California}

\abstract{Reasoning models have advanced rapidly, but the dominant reinforcement learning from verifiable rewards (RLVR) recipe remains surprisingly narrow: sample many responses and reward each with a single bit indicating whether the final answer is correct. Yet many settings provide \textit{rich feedback}, including execution traces, tool outputs, expert corrections, and model self-evaluations. We study how to use such feedback through a distributional variant of the classic imitation learning algorithm DAgger, where the learner has local access to an expert distribution on states visited by the current policy. This yields a simple forward cross-entropy objective that admits a blackbox expert and whose sequence-level gradient {conduct rich credit assignment by propagating} future expert-student disagreement back to earlier decisions. We show that prior RL with self-distillation objectives based on reverse KL or Jensen-Shannon fail to guarantee monotonic policy improvement: even when the expert has higher reward, their updates may increase probability on worse actions. In contrast, we show that forward cross-entropy admits monotonic policy improvement and enjoys guarantees on regret. We further show that our objective optimizes a lower bound on teacher-weighted likelihood of success, leading to improved Pass@N. Empirically, our approach, DistIL, improves over RLVR and RL with self-distillation baselines across a variety of domains: scientific reasoning, coding, and solving hard mathematical problems.}

\date{June 3, 2026}
\website{https://rishabh-1086.github.io/project-distIL}
\code{https://github.com/rishabh-1086/distIL}
\correspondence{{\email{\{rishabha, feinashl, paria.rashidinejad\}@usc.edu}}}

\begin{document}

\maketitle

\begin{figure}[htbp]
    \centering

    \begin{subfigure}{0.245\textwidth}
        \includegraphics[width=\linewidth]{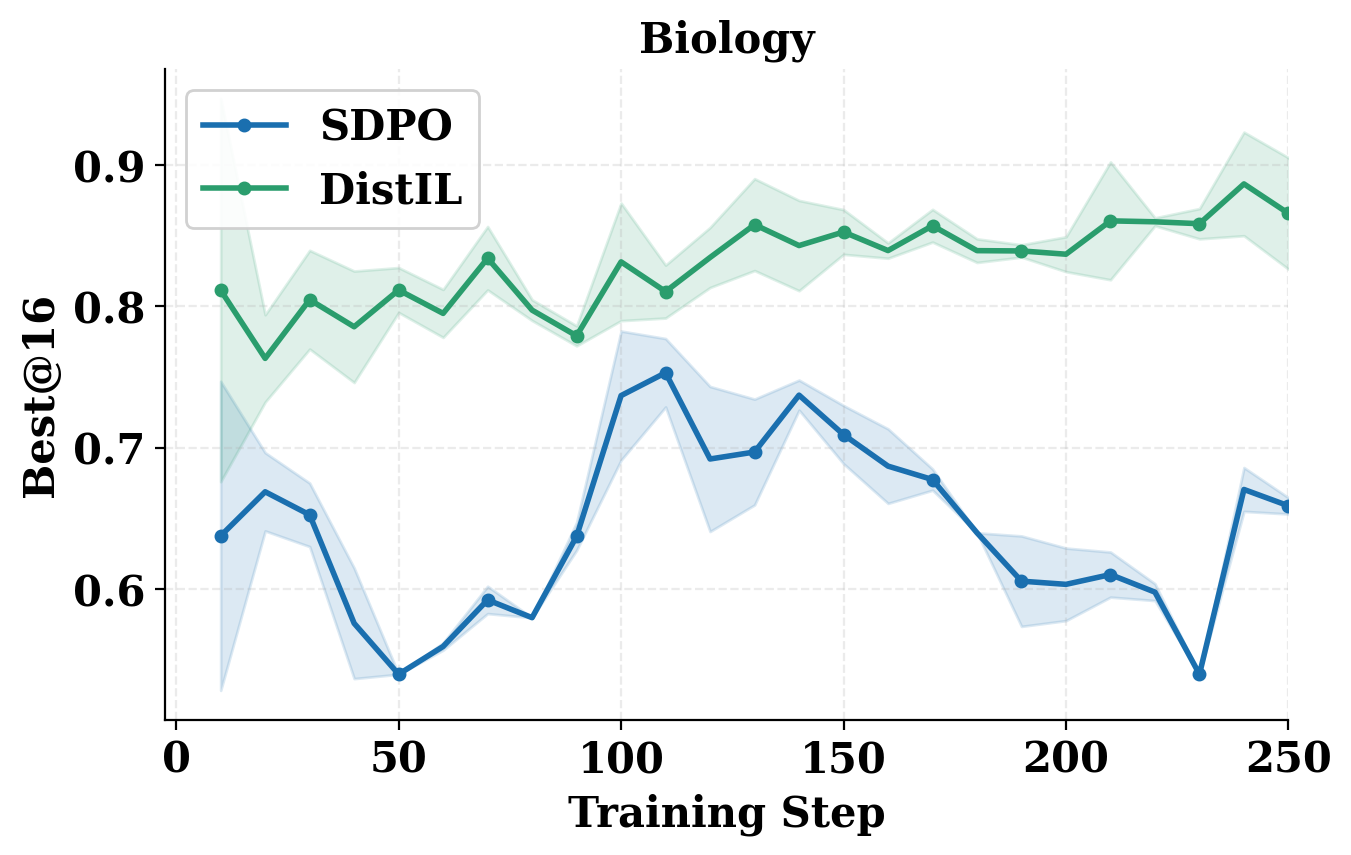}
    \end{subfigure}
    \begin{subfigure}{0.245\textwidth}
        \includegraphics[width=\linewidth]{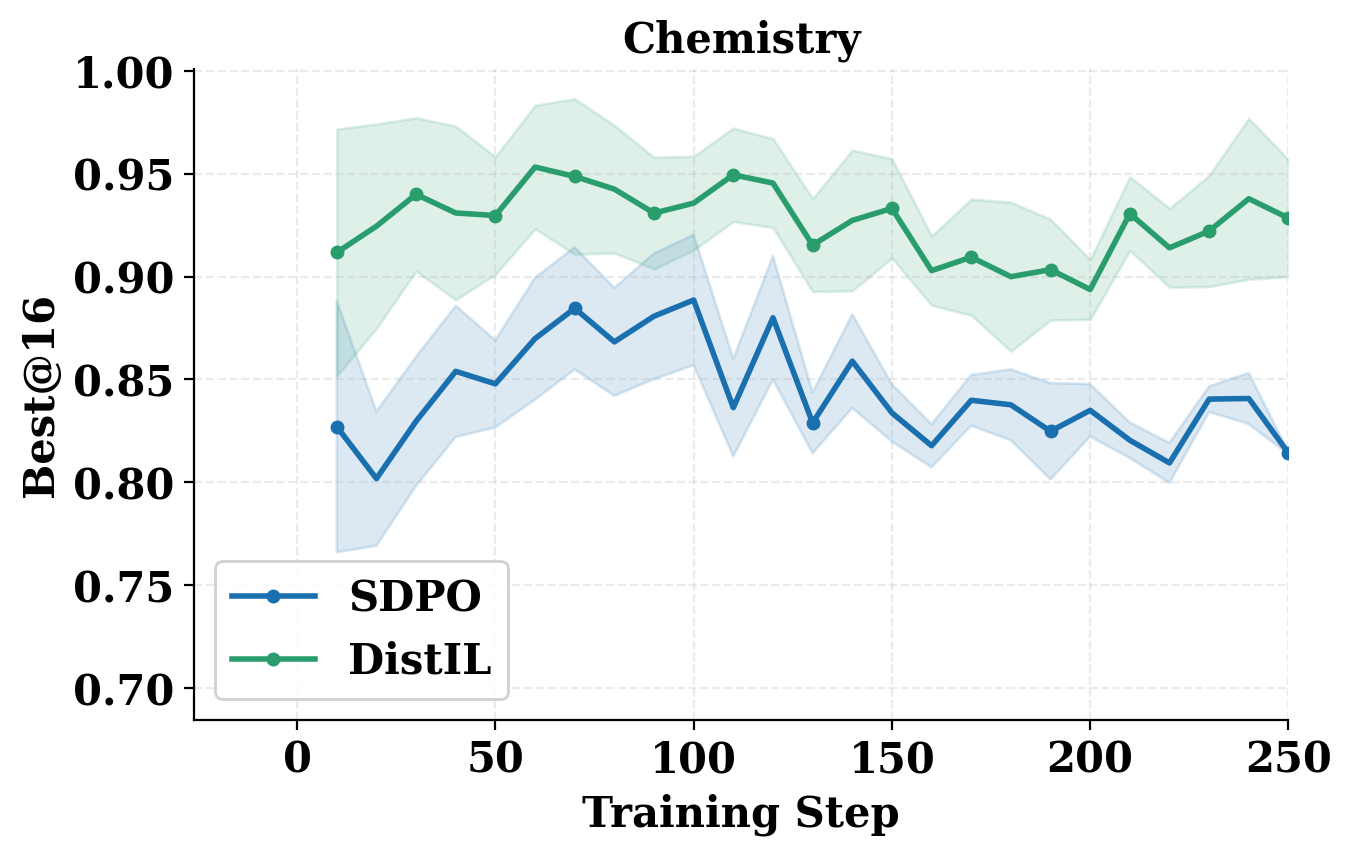}
    \end{subfigure}
    \begin{subfigure}{0.245\textwidth}
        \includegraphics[width=\linewidth]{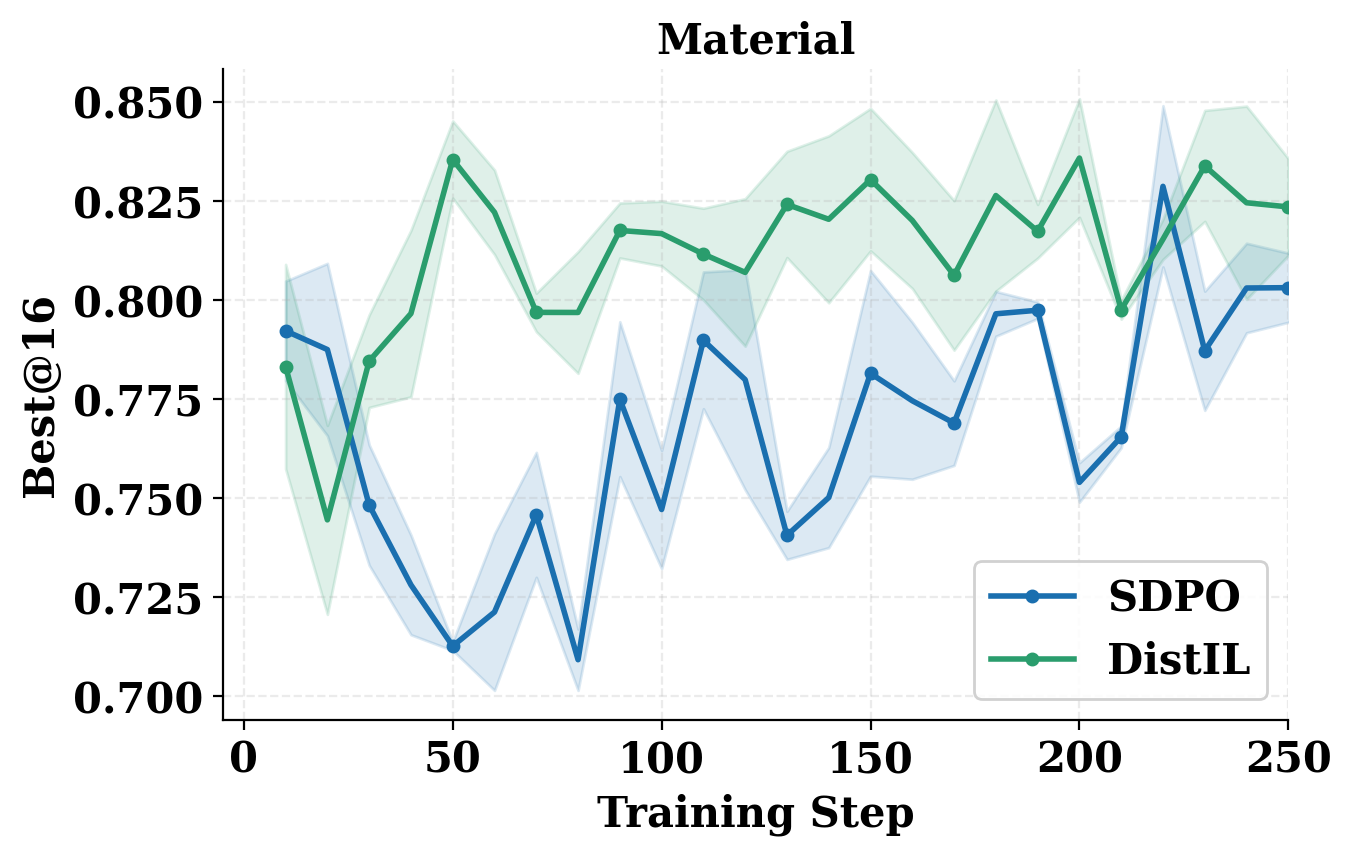}
    \end{subfigure}
    \begin{subfigure}{0.245\textwidth}
        \includegraphics[width=\linewidth]{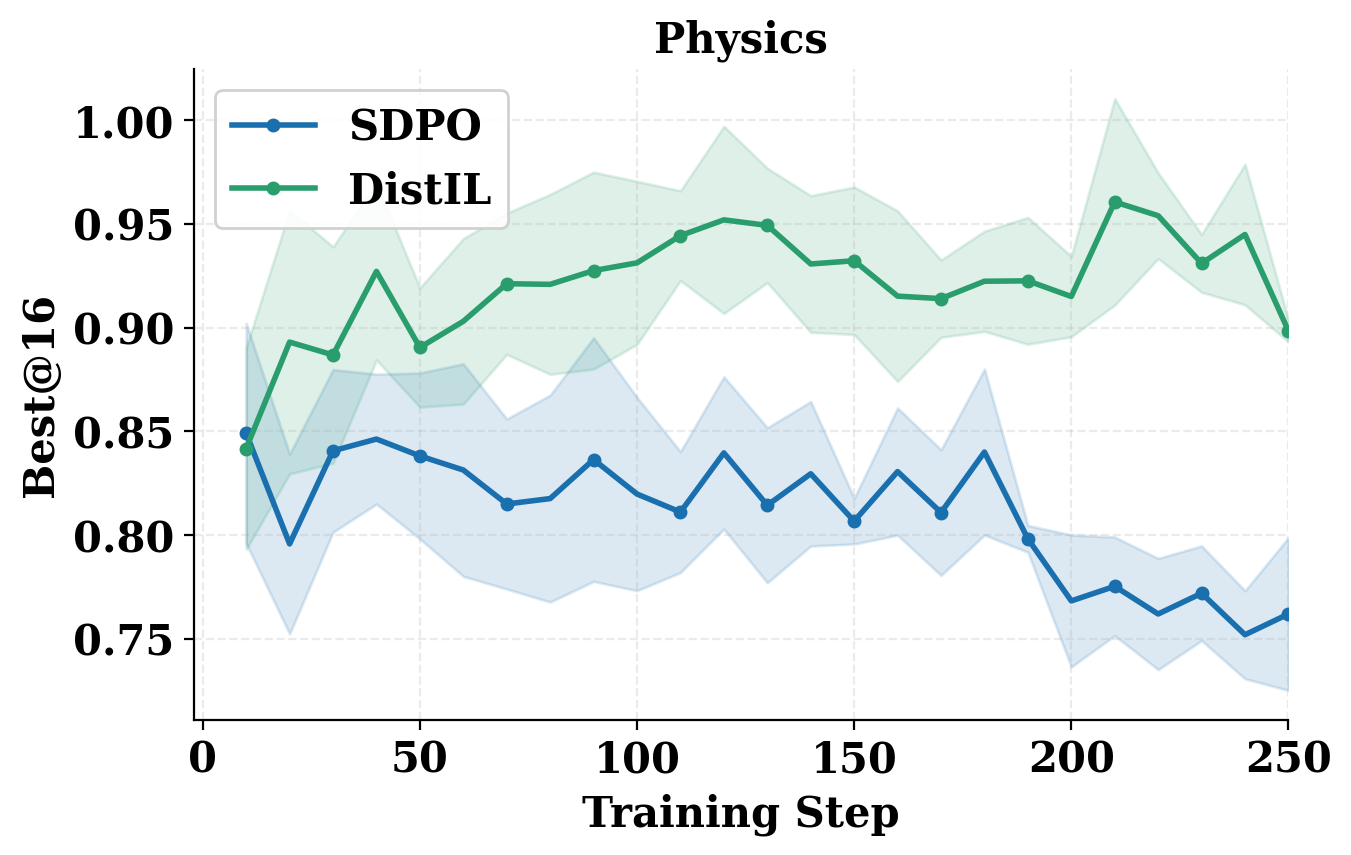}
    \end{subfigure}

    \vspace{0.5em}

    \begin{subfigure}{0.245\textwidth}
        \includegraphics[width=\linewidth]{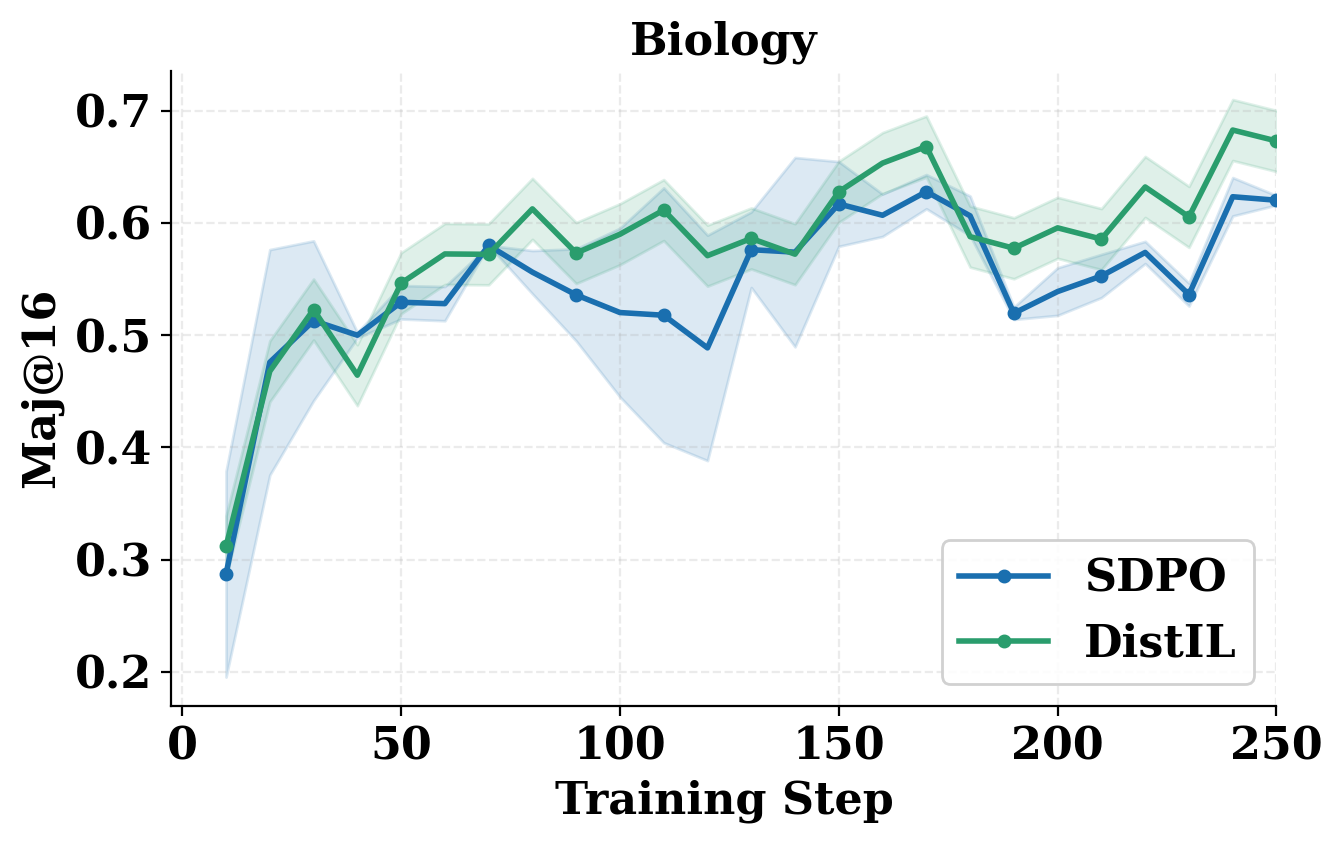}
    \end{subfigure}
    \begin{subfigure}{0.245\textwidth}
        \includegraphics[width=\linewidth]{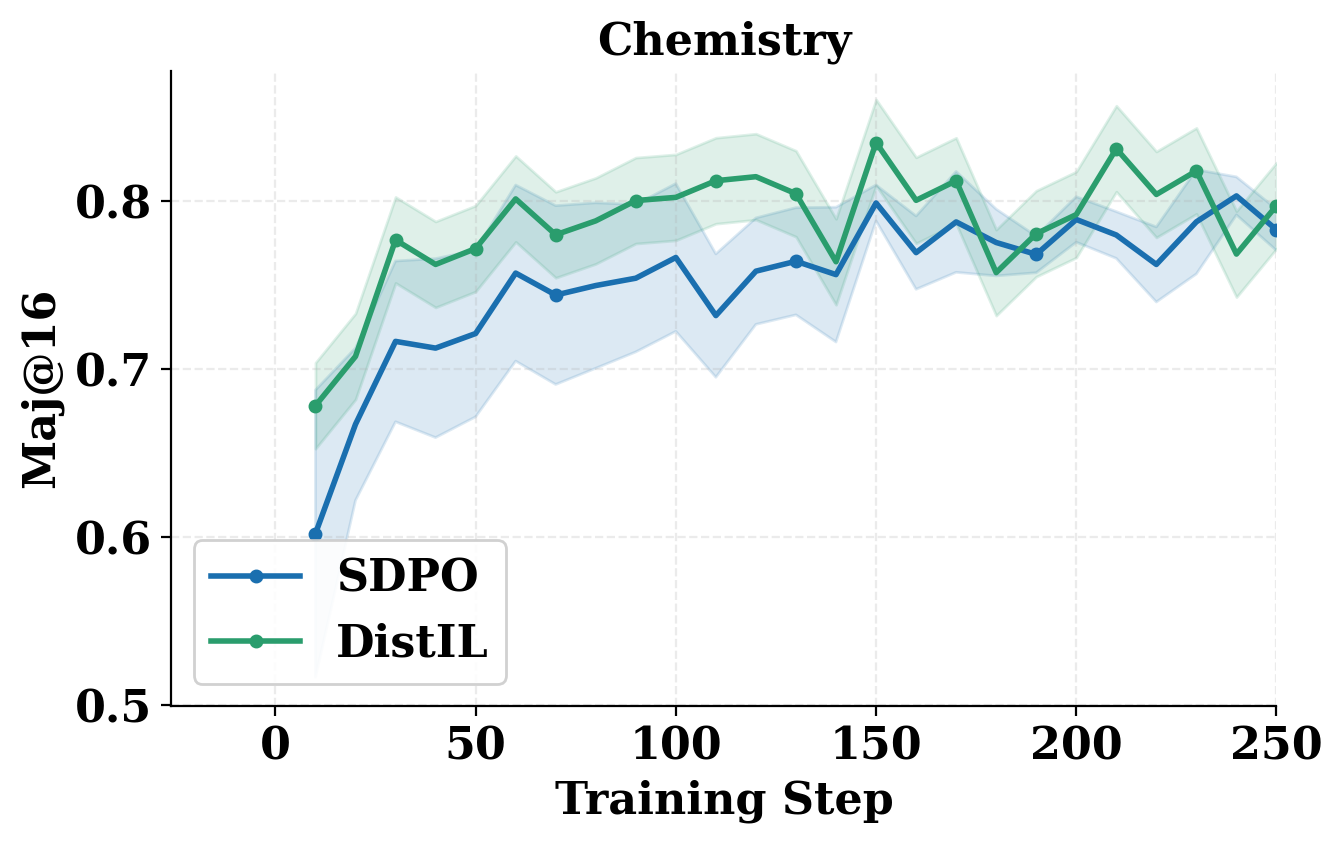}
    \end{subfigure}
    \begin{subfigure}{0.245\textwidth}
        \includegraphics[width=\linewidth]{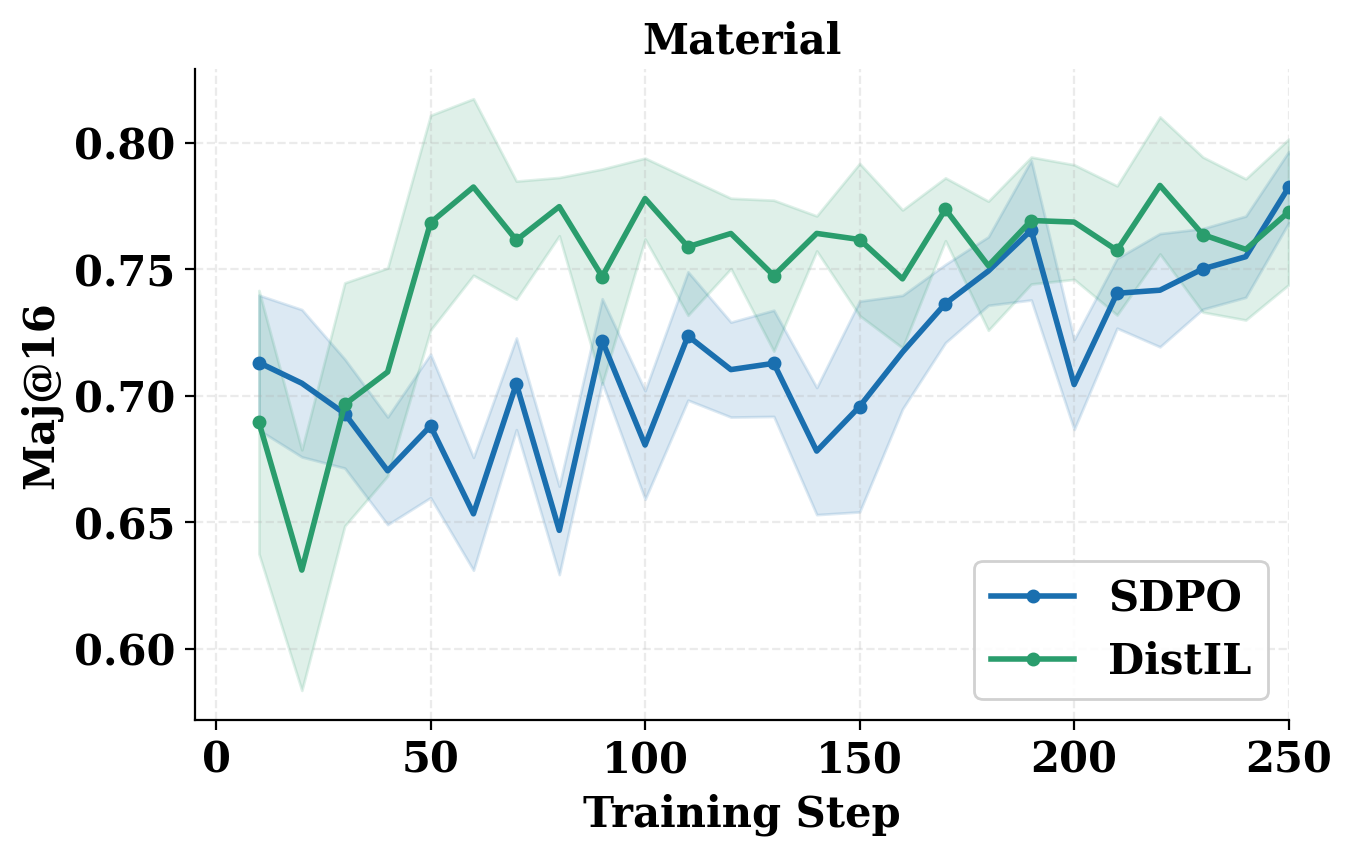}
    \end{subfigure}
    \begin{subfigure}{0.245\textwidth}
        \includegraphics[width=\linewidth]{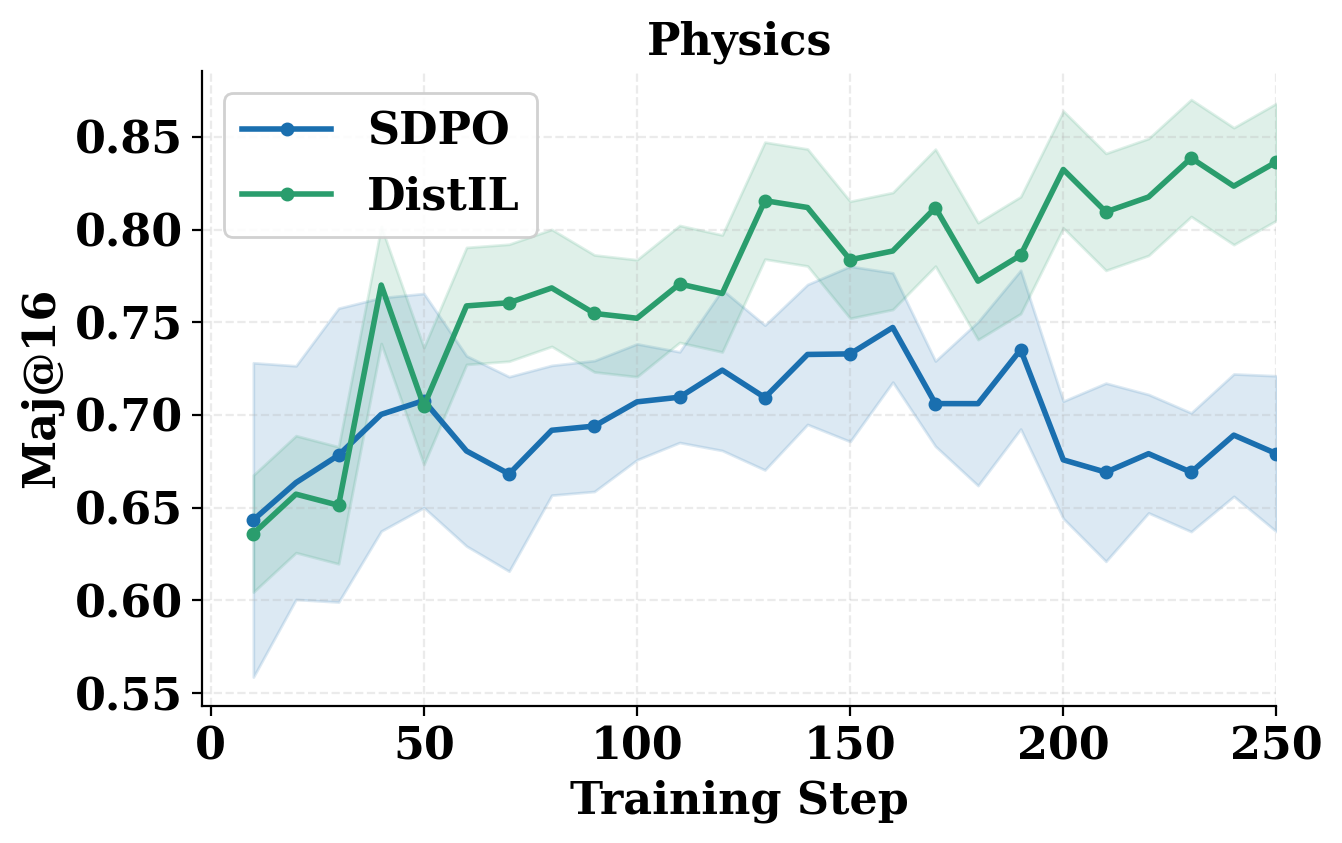}
    \end{subfigure}
    \vspace{0.5em}
    \caption{Validation Best@16 (top) and Maj@16 (bottom) over training for RL with self-distillation algorithm SDPO~\citep{hubotter2026reinforcement} and DistIL (ours) on Qwen3-8B across four scientific reasoning domains: biology, chemistry, materials, and physics. DistIL generally achieves higher validation performance than SDPO across domains and metrics, with gains often appearing early and sustained during training. SDPO exhibits greater variability with longer training, including a pronounced decline in biology Best@16 after roughly 100 steps and larger oscillations in chemistry and physics; DistIL is comparatively more stable.}
    \label{fig:science_complete_metrics}
\end{figure}

\section{Introduction}
\label{sec:intro}

Reasoning large language models have made significant recent progress \citep{jaech2024openai, guo2025deepseek, comanici2025gemini, olmo2025olmo}. The dominant paradigm, reinforcement learning from verifiable rewards (RLVR) improves reasoning by sampling many candidate solutions from the model and providing each with a single bit of feedback: whether the final answer is correct. This simple paradigm has been remarkably successful in domains such as coding and mathematical and scientific reasoning tasks where final answers can be checked automatically~\citep{tan2025deepscaler, yue2025vapo, zeng2025simplerl, pan2025metaspatial, luo2025deepcoder}. However, RLVR has important limitations. A scalar terminal reward is typically broadcast across the entire response, causing every token to receive the same weight regardless of its actual contribution. This provides little credit assignment: it does not reveal which tokens, phrases, or reasoning steps caused the success or failure. Moreover, the paradigm's reliance on automatic final-answer verification makes it difficult to apply in unverifiable domains, where final answer correctness cannot be easily checked.

Many important settings, however, provide feedback that is far richer than a single correctness bit: error logs, unit-test results, natural-language critiques, hints, or ground-truth responses. Inspired by the seminal work on knowledge distillation \citep{hinton2015distilling}, a growing line of research seeks to exploit these signals through on-policy self-distillation, in which the current model is conditioned on feedback to produce a privileged teacher, whose distribution is then distilled back into the original policy~\citep{hubotter2026reinforcement, zhao2026self}. This turns sparse outcome supervision into dense token-level guidance. In particular, viewed through the lens of policy gradients, these objectives induce an implicit form of credit assignment: different tokens receive different weights, e.g., determined by the local discrepancy between the student and feedback-conditioned teacher policies~\citep{hubotter2026reinforcement}.

\subsection{Contributions}

This motivates us to take a closer look at on-policy self-distillation. Many classical policy-optimization methods---such as conservative policy iteration~\citep{kakade2002approximately}, natural policy gradient~\citep{kakade2001natural}, and trust-region policy optimization~\citep{schulman2015trust}---are designed around \emph{monotonic policy improvement}: under an idealized update, such as a sufficiently small natural-gradient or trust-region step, each update should increase policy's average reward. This raises a natural question for on-policy self-distillation. If a feedback-conditioned teacher achieves higher average reward than the student, should an idealized update that pulls the student toward the teacher improve the student? We show that, for on-policy self-distillation objectives with $f$-divergences, the answer is in general, no. To address this, we revisit the classical DAgger framework~\citep{ross2011reduction} and propose a distributional variant that enjoys monotonic policy-improvement and regret guarantees, while empirically outperforming existing methods. Concretely, we make the following contributions:

\textbf{Provable limitations of existing on-policy self-distillation.}
We begin by analyzing on-policy self-distillation objectives based on general $f$-divergences, a class that subsumes existing objectives such as reverse-KL divergence~\citep{hubotter2026reinforcement} and Jensen--Shannon divergence~~\citep{zhao2026self}. We identify two limitations. {First,} we prove that optimizing an $f$-divergence does not, in general, guarantee monotonic policy improvement; even when the teacher is strictly better than the student and the update is an idealized natural-gradient step, the objective can increase probability on suboptimal actions (Propositions~\ref{prop:f_divergence_non_monotonic}, \ref{prop:reverse_kl_non_improvement_counterexample}). {Second,} existing methods typically rely on approximate gradients that provide only \emph{local} credit assignment, measuring teacher--student discrepancy at individual tokens. We show that this local approximation ignores how early token choices shape the future states where teacher--student disagreement appears, erasing precisely the delayed feedback needed to credit earlier decisions. We prove that this can cause learning to converge to strictly suboptimal policies (Proposition~\ref{prop:grad_issue}).

\textbf{DistIL: A new algorithm for RL with rich feedback.}
To address these limitations, we propose DistIL, a distributional variant of the DAgger imitation-learning algorithm. DistIL optimizes a forward cross-entropy loss between the feedback-conditioned teacher and the student policy, on states visited by the student (c.f. Equation \eqref{eqn:our_objective_final}). Unlike reverse-KL or Jensen--Shannon self-distillation objectives, DistIL can leverage teacher log-probabilities when available, but does not require them and its objective can be estimated using samples from the teacher alone, and therefore naturally accommodates black-box teachers such as external models or human experts. Furthermore, DistIL performs \emph{future-aware} credit assignment: teacher--student disagreement at later timesteps can be propagated back to the earlier decisions (c.f. Equation \eqref{eqn:grad_expression}).

\textbf{DistIL enjoys theoretical guarantees.}
Beyond algorithmic simplicity, DistIL comes with several theoretical guarantees that explain why the objective is well aligned with reward improvement. First, we show that DistIL's forward cross-entropy loss guarantees monotonic-improvement property missing from prior self-distillation objectives: under standard local assumptions, moving toward a better teacher improves the student's expected reward (Proposition~\ref{prop:monotonic_main}). Second, we prove that DistIL achieves sublinear regret (Theorem~\ref{thm:suboptimality_d_dagger}). Finally, we show that DistIL maximizes a teacher-weighted lower bound on the expected log-likelihood of success, drawing connections to maximum-likelihood RL~\citep{tajwar2026maximum} and providing a principled explanation for improvements in \texttt{pass@$N$} for any $N$ (Proposition~\ref{prop:weighted_connection}, c.f. Figure \ref{fig:science_complete_metrics}).

\textbf{Empirical results.} The theoretical advantages of DistIL translate into consistent empirical gains. Across scientific reasoning~\citep{feng2024sciknoweval}, coding~\citep{jain2024livecodebench}, and challenging mathematical reasoning benchmarks, DistIL outperforms strong RL and self-distillation baselines, including SDPO~\citep{hubotter2026reinforcement}, OPSD~\citep{zhao2026self}, and GRPO~\citep{shao2024deepseekmath}. The gains hold across different feedback regimes, from sparse final-answer correctness to execution traces and ground-truth solutions, highlighting the practical benefit of pairing a reward-aligned objective with future-aware credit assignment.
\section{Background and Problem Formulation}
\label{sec:prelim}

\noindent \textbf{Contextual Markov decision process and reinforcement learning.}
We model consider finite-horizon contextual Markov decision processes as the environment. At the beginning of each episode, a context (i.e., a prompt/problem) $x \sim \rho$ is sampled. The initial state is set to be the context $s_1=(x,\emptyset)$. At each timestep $t \in [H]$, the policy selects an action $y_t \sim \pi(\cdot \mid s_t)$, and the environment transitions to the next state. In the autoregressive language-model setting considered in this work, the transition is deterministic and the next state is obtained by appending the sampled action to the current prefix. Thus, we write $s_t=(x,y_{1:t-1})$. For simplicity, we sometimes write the entire trajectory as $y=(y_1,\ldots,y_H)$. 

Given per-step rewards $r_t(x,y_{1:t})$, the goal of reinforcement learning is to learn a policy maximizing expected cumulative reward:
\begin{equation}
J(\pi)
:=
\mathbb{E}_{x \sim \rho}
\mathbb{E}_{y \sim \pi(\cdot \mid x)}
\left[
\sum_{t=1}^{H} r_t(x,y_{1:t})
\right].
\label{eqn:rl_objective}
\end{equation}

\noindent \textbf{Reinforcement learning with verifiable rewards (RLVR).}
In many reasoning tasks, rewards can be obtained from automatic verifiers: code can be checked with unit tests, while in some mathematical and scientific problems final answer correctness can be checked. This gives rise to reinforcement learning with verifiable rewards (RLVR), where supervision is typically sparse and outcome-based: $r_t=0$ for $t<H$, and the terminal reward $r_H$ evaluates the final answer, often as a binary correctness signal. Methods such as GRPO~\citep{shao2024deepseekmath} and its variants~\citep{liu2025understanding, ahmadian2024back} optimize policies using such rewards. Despite its simplicity and empirical success, RLVR provides only delayed supervision. Furthermore, the terminal reward is broadcast across the response, giving every token the same weight regardless of its role in the reasoning process. This makes credit assignment difficult. Moreover, in group-relative or advantage-based updates~\citep{shao2024deepseekmath}, the learning signal can vanish when all sampled responses receive identical outcomes, such as all correct or all incorrect.

\noindent \textbf{Learning from rich feedback via distillation.}
In many settings, supervision extends beyond final-answer correctness and includes richer signals such as execution traces, error messages, expert annotations, or feedback from other models. These signals provide more informative guidance and can be leveraged to improve learning. Inspired by knowledge distillation, a recent line of work studies variants of \emph{on-policy distillation}, where a student policy $\pi_S$ is trained to match a teacher policy $\pi_T$ along trajectories sampled from the student~\citep{agarwal2024policy, yang2025qwen3, xiao2026mimo, zeng2026glm}. In on-policy \emph{self}-distillation, the teacher is derived from the same model but conditioned on additional feedback $f$:
\[
\pi_T(\cdot \mid x, y_{1:t-1}) := \pi_\theta(\cdot \mid x, f, y_{1:t-1}), \quad
\pi_S(\cdot \mid x, y_{1:t-1}) := \pi_\theta(\cdot \mid x, y_{1:t-1}),
\]
where $\pi_T$ denotes the feedback-conditioned teacher and $\pi_S$ denotes the student policy. Given trajectories $y \sim \pi_S(\cdot \mid x)$, training minimizes
\begin{equation}
\mathcal{L}_{\text{self-distill}} :=
\mathbb{E}_{x \sim \rho,\; y \sim \pi_S(\cdot \mid x)}
\left[
\sum_{t=1}^{|y|}
\mathsf{D}\big(
\pi_S(\cdot \mid x, y_{1:t-1})
\;\|\;
\pi_T(\cdot \mid x, y_{1:t-1})
\big)
\right],
\label{eqn:general_opd_loss_distillation}
\end{equation}
where $\mathsf{D}$ is a divergence measure, such as reverse KL in SDPO~\citep{hubotter2026reinforcement} or Jensen--Shannon divergence in OPSD~\citep{zhao2026self}.

\noindent \textbf{Problem formulation: RL from rich feedback.}
We now formalize the reinforcement-learning problem considered in this paper. We work with the deterministic contextual MDP introduced above, and let $\pi_\theta \in \Pi$ denote a policy parameterized by $\theta$. As in the self-distillation setting, at each training round rich feedback $f$ induces a privileged teacher policy $\pi_T$, which is held fixed during the student update. Our central assumption is that this feedback is useful: for some underlying, possibly unknown reward function, the feedback-induced teacher is at least as good as the current unconditioned student in expected reward. Formally, for the student policy $\pi_\theta$, we assume
\begin{align}\label{eq:teacher_student_gap}
\Delta \coloneqq
\E_{x \sim \rho,, y \sim \pi_T(\cdot \mid x)} \left[r(x,y) \right]
-
\E_{x \sim \rho,, y \sim \pi_\theta(\cdot \mid x)}\left[r(x,y) \right]
\geq 0.
\end{align}
For example, the above assumption says that a model conditioned on execution logs, critiques, hints, or a correct solution should, on average, achieve a higher success rate than the same model without access to that feedback. This assumption is precisely what makes the self-distillation paradigm meaningful; if the privileged teacher were not better aligned with the task reward than the student, then distilling the teacher would have no reason to improve the student. Since the reward $r$ may be unknown, the learning algorithms do not optimize the reward objective directly and they instead optimize distillation-based surrogate losses.

\noindent \textbf{Notation.}
For distributions $p$ and $q$, we write $\mathrm{KL}(p\|q) \coloneqq \mathbb{E}_{y \sim p}\left[\log {p(y)}/{q(y)}\right]$ to denote the Kullback--Leibler divergence and $\mathrm{JS}(p\|q) \coloneqq \frac{1}{2}\mathrm{KL}(p\|m)+\frac{1}{2}\mathrm{KL}(q\|m)$ for Jensen--Shannon divergence, where $m \coloneqq \frac{1}{2}(p+q)$. We also write $\mathrm{H}^{\times}(p,q) \coloneqq -\mathbb{E}_{y \sim p}[\log q(y)]$ for forward cross-entropy. The operator $\mathsf{sg}(\cdot)$ denotes stop-gradient. We write $x \lesssim y$ when there exists a  constant $c>0$ such that $x \leq c y$.
\section{Limitations of Existing On-Policy Self-Distillation Methods}
\label{sec:problem}

The formulation above raises the following question: if rich feedback induces a better teacher, does self-distillation necessarily turn this advantage into policy improvement? In this section, we start by analyzing the two consequential choices made in existing methods: the divergence measure used to distill the teacher, and the gradient estimator used to optimize it. We show that both choices can fail in principle, even in favorable, idealistic scenarios. First, we ask whether divergence minimization guarantees monotonic policy improvement when the teacher is better than the student (Section~\ref{sec:policy_improvement}). Second, we show that the local, tokenwise gradients used in prior works \citep{zhao2026self, hubotter2026reinforcement} can ignore how early decisions affect future teacher--student disagreement, leading to strictly suboptimal policies (Section~\ref{sec:local_credit_assignment}).

\subsection{Does divergence minimization guarantee monotonic policy improvement?}\label{sec:policy_improvement}

A useful property for policy optimization algorithms is \emph{monotonic policy improvement}: under an idealized update, the policy should move in a direction that increases expected reward. This principle underlies classical methods such as natural policy gradient (NPG) \citep{kakade2001natural}, and is approximately enforced by widely used trust-region and proximal methods such as TRPO \citep{schulman2015trust}, PPO \citep{schulman2017proximal}, and GRPO \citep{shao2024deepseekmath}. This motivates the analogous question for self-distillation: if the teacher has higher expected reward than the student, does an ideal update that pushes the student toward the teacher necessarily improve the student? We answer this question for general $f$-divergence self-distillation objectives, a class that includes reverse KL as used in SDPO \citep{hubotter2026reinforcement} and Jensen--Shannon divergence as used in OPSD \citep{zhao2026self}. The proposition below shows that, even in contextual bandits, minimizing such objectives does not in general guarantee monotonic policy improvement.

\begin{proposition}[\textbf{$f$-divergence self-distillation does not guarantee monotonic policy improvement}]
    \label{prop:f_divergence_non_monotonic}
    Let $\pi_T$ be a fixed teacher and $\pi_\theta$ a student policy. For the softmax parameterization, the Natural Policy Gradient (NPG) step minimizing an $f$-divergence $D_f(\pi_T \| \pi_\theta)$, i.e., $
    \theta' = \theta - \eta \, F(\theta)^{-1} \nabla_{\theta} D_f(\pi_T \| \pi_\theta)$,
    satisfies
    \begin{align*}
        J(\pi_{\theta^{\prime}})
        =
        J(\pi_{\theta})
        -
        \eta \, \Cov_{y \sim \pi_{\theta}}
        \left(
            r(y),
            g(y)
        \right)
        +
        O(\eta^2),
    \end{align*}
    where $g(y)
    =
    f\!\left(\frac{\pi_T(y)}{\pi_\theta(y)}\right)
    -
    \frac{\pi_T(y)}{\pi_\theta(y)}\, f'\!\left(\frac{\pi_T(y)}{\pi_\theta(y)}\right)$. Consequently, for sufficiently small $\eta > 0$, the NPG step provides a policy improvement if and only if $\Cov_{y \sim \pi_{\theta}} \left( r(y), g(y) \right) < 0$. In particular, a positive teacher--student reward gap $\Delta > 0$ alone does not guarantee monotonic policy improvement.
\end{proposition}

A detailed proof can be found in Appendix \ref{subsec:app_monotonic_f_div}. Proposition~\ref{prop:f_divergence_non_monotonic} highlights a simple but important distinction: a teacher can be better on average, while the update induced by a distillation objective can still point in a reward-decreasing direction. For SDPO, which corresponds to reverse KL, the first-order change in reward is governed by
\begin{align*}
\Cov_{y\sim\pi_\theta} \left(r(y), \log\frac{\pi_\theta(y)}{\pi(y)}\right).
\end{align*}
This covariance measures how the student--teacher mismatch aligns with reward. Importantly, the covariance is \emph{not} determined by the teacher's average advantage $\Delta$, defined in \eqref{eq:teacher_student_gap}. Instead, it depends on \emph{where} the student and teacher differ. Reverse KL improves the student only when the trajectories over-weighted by the student relative to the teacher are, on average, low-reward. If the student over-weights trajectories that are actually high-reward, the same update may suppress good behavior and decrease expected reward.

The failure can already occur in a three-action bandit. Imagine three trajectories: excellent ($r=1$), mediocre ($r=0.5$), and bad ($r=0$). The teacher may be better overall because it puts more probability on the excellent trajectory, but it also puts far less probability than the student on the mediocre trajectory. A reverse-KL update then treats the student's extra mass on the mediocre trajectory as an error to be corrected; after normalization, this can increase probability on the bad trajectory even though it has zero reward. Thus, being better on average is not enough and the distillation direction must also be aligned with reward. The next proposition makes this 
construction concrete and formally proves the failure result for reverse-KL distillation.

\begin{proposition}[\textbf{Reverse-KL distillation can decrease reward}]
\label{prop:reverse_kl_non_improvement_counterexample}
There exist universal constants $c_0,c_1>0$ and a three-armed bandit with rewards in $[0,1]$, a fixed teacher policy $\pi_T$, and a student policy $\pi_\theta$ such that the teacher is better than the student by a constant margin $\Delta \geq c_0 > 0$ but the natural policy gradient step minimizing the reverse KL objective, decreases the student's expected reward. In particular, for all sufficiently small $\eta>0$, we have $J(\pi_{\theta'}) \leq J(\pi_\theta)-c_1\eta.$
\end{proposition}

The detailed construction is given in Appendix~\ref{subsec:reverse_kl_failure_example}. It uses a three-action bandit where the teacher has a constant reward advantage over the student. Nevertheless, due to the normalization in the reverse-KL NPG update, the update can increase probability on a low-reward action whose student--teacher log-ratio is below average. Thus, expected reward can decrease even though the teacher is better overall.

\begin{AIbox}{Takeaway: Self-distillation with $f$-divergence does not guarantee monotonic policy improvement.}
For $f$-divergence distillation, policy improvement depends on whether the induced update direction is aligned with reward, not only on whether the teacher has higher expected reward. Thus, even when the teacher is strictly superior to the students, updates induced by reverse KL in SDPO and Jensen--Shannon divergence in OPSD can degrade the student.
\end{AIbox}

\subsection{Local credit assignment leads to suboptimal policies}\label{sec:local_credit_assignment}

Next, we examine the gradient approximations used in existing self-distillation methods such as SDPO and OPSD. In the self-distillation objective \eqref{eqn:general_opd_loss_distillation}, the student policy $\pi_\theta$ appears in two places: inside the divergence, and in the trajectory distribution that determines which states are visited. Keeping the teacher $\pi_T$ fixed, the full gradient of this objective decomposes into a local term and a future-credit term:
\begin{align*}
    \nabla_{\theta}\mathcal{L}_{\text{self-distill}}
    =
    \underbrace{
    \E \left[
    \sum_{t=1}^{|y|}
    \nabla_\theta D(\pi_\theta(\cdot \mid s_t) \| \pi_T(\cdot \mid s_t))
    \right]
    }_{\coloneqq \nabla_{\text{local}}}
    +
    \underbrace{
    \E \left[
    \sum_{t=1}^{|y|}
    \nabla_\theta \log \pi_\theta(y_t \mid s_t)
    \left(
    \sum_{\tau \geq t}
    D(\pi_\theta(\cdot \mid s_\tau) \| \pi_T(\cdot \mid s_\tau))
    \right)
    \right]
    }_{\coloneqq \nabla_{\text{future}}},
\end{align*}
where $s_t = (x, y_{1:t-1})$. The first term is the tokenwise distillation gradient, which updates the policy using only the local mismatch between the student and teacher at the current state. For example, under reverse KL, this term weights $\nabla_\theta \log \pi_\theta(y \mid s_t)$ by the local log-ratio $\log \frac{\pi_T(y \mid s_t)}{\pi_\theta(y \mid s_t)}$, yielding an implicit form of token-level credit assignment.

SDPO and OPSD effectively approximate the full gradient by retaining only this local term, $\nabla_{\theta}\mathcal{L}_{\text{self-distill}} \approx \nabla_{\text{local}}$. This drops the second term, which is precisely a credit-to-go term: it assigns credit to an action at step $t$ according to the teacher--student mismatch that appears at future states reached after taking that action. Without this term, the update can miss delayed feedback entirely, assigning no credit to early choices whose consequences appear only downstream. The next proposition proves that this local approximation can lead to strictly suboptimal policies.

\begin{proposition}[\textbf{Local credit assignment can be strictly suboptimal}]
\label{prop:grad_issue}
There exist a two-step contextual MDP with action space $\mathcal{A} = \{a,b,c\}$, a parametric student policy class $\{\pi_{u,v}\}_{u,v\in\mathbb{R}}$, a teacher policy $\pi_T$, and a reward function $r$ such that, starting from initialization $(u,v)=(0,0)$, the local credit assignment gradient converges to a policy $\pi_{\mathrm{local}}$ with $J(\pi_{\mathrm{local}})=\frac13$, whereas the full sequence-level gradient converges to a policy $\pi_{\mathrm{seq}}$ with
$J(\pi_{\mathrm{seq}})=\frac25$. Consequently, $J(\pi_{\mathrm{seq}}) > J(\pi_{\mathrm{local}})$, and local credit assignment is strictly suboptimal.
\end{proposition}

\noindent
Proposition~\ref{prop:grad_issue} highlights a fundamental limitation of local tokenwise
updates: they may fail to assign credit to earlier decisions, leading to strictly
suboptimal solutions. Consequently, methods such as SDPO and OPSD, which rely on such
local gradient estimators, can fail to learn from delayed feedback.

The construction in Appendix~\ref{subsec:app_local_gradient_issue} considers a two-step problem in which the teacher--student mismatch is encountered only along one branch of the decision tree. Under the local objective, the contribution of this mismatch is differentiated while treating the prefix distribution as fixed, giving zero gradient with respect to the first-step decision. As a result, optimization remains at the symmetric initialization and produces a policy with expected reward $1/3$. In contrast, the full sequence-level objective accounts for how the first-step decision affects the distribution of future states. This induces a nonzero gradient on the first-step parameter, allowing optimization to change the first-step decision and shift probability mass toward higher-reward actions and attain expected reward $2/5$. Thus, local credit assignment can converge to a strictly suboptimal policy even in a simple two-step problem.

\begin{AIbox}{Takeaway: Tokenwise local gradient approximation can fail to credit early decisions.}
SDPO and OPSD apply a stop-gradient on the sampling process, so gradients pass only through the per-token divergence and may ignore how current token choices affect future token distributions. This can yield zero gradient for consequential early decisions while the full sequence-level gradient is nonzero, causing convergence to strictly suboptimal policies.
\end{AIbox}

Taken together, these results highlight two distinct limitations of existing self-distillation approaches: the objective may induce updates that are not monotonically reward-improving, and local tokenwise gradient approximations may miss the delayed consequences of early decisions. This suggests that learning from rich feedback requires both a more reward-aligned objective and a gradient estimator that preserves sequence-level credit assignment. We develop such an approach next.
\section{RL from Rich Feedback with Distributional Imitation Learning}
\label{sec:method}

\noindent \textbf{Distributional imitation learning formulation.} We view self-distillation through the lens of \emph{on-policy imitation learning}, in the spirit of DAgger \citep{ross2011reduction}: the learner trains on states induced by its own policy, while an expert provides guidance on those visited states. This perspective is especially natural for learning from rich feedback, where the student generates trajectories, and feedback induces a privileged teacher that can be queried on the states the student encounters.

We develop a distributional variant of DAgger. In contrast to the standard setup where the expert supplies a single action at each visited state, we allow the supervision to take the form of a local teacher distribution: when available, we use $\pi_T(\cdot \mid s_t)$; when only samples are available, we estimate the objective using one or more samples $y_t \sim \pi_T(\cdot \mid s_t)$. Concretely, we propose the following distributional imitation-learning (DistIL) objective:
\begin{equation}
\mathcal{L}_{\mathrm{DistIL}}(\theta)
:=
\mathbb{E}_{x \sim \rho,\; y \sim \pi_\theta(\cdot \mid x)}
\left[
\sum_{t=1}^{H}
H^{\times}\big(
\pi_T(\cdot \mid x, y_{1:t-1}),
\;\pi_\theta(\cdot \mid x, y_{1:t-1})
\big)
\right]
\label{eqn:our_objective_final}
\end{equation}
where we use $\pi_T(\cdot \mid x, y_{1:t-1}) = \mathsf{sg}\left(\pi_{\theta}(\cdot \mid x, y_{1:t-1}, f)\right)$. Importantly, this objective does not \emph{require} access to the teacher probabilities. It admits unbiased sample-based estimates, while also allowing teacher probabilities to be used when available. This flexibility allows DistIL to accommodate both black-box teachers, such as external models or human experts, as well as feedback-conditioned privileged teachers. This gives DistIL a practical advantage over SDPO, whose objective requires access to teacher probabilities. 

\noindent \textbf{Full gradients enabling future-aware credit assignment.} Motivated by the failure of local gradient approximations established in Section \ref{sec:local_credit_assignment}, we optimize $\mathcal{L}_{\mathrm{DistIL}}(\theta)$ defined in \eqref{eqn:our_objective_final} using full gradients rather than treating the student-induced state distribution as frozen. In Appendix \ref{subsec:app_grad_derivation}, we show that the gradients of DistIL objectives are given by:
\begin{equation}
\begin{split}
\nabla_\theta \mathcal{L}_{\mathrm{DistIL}}
&=
\underbrace{
\mathbb{E}_{x \sim \rho, y \sim \pi_\theta(\cdot \mid x)}
\left[
\sum_{t=1}^{H}
\nabla_\theta 
H^{\times}\Big(
\pi_T(\cdot \mid s_t),
\;\pi_\theta(\cdot \mid s_t)
\Big)
\right]
}_{\text{local credit assignment}}
\\
&\quad+
\underbrace{
\mathbb{E}_{x \sim \rho, y \sim \pi_\theta(\cdot \mid x)}
\left[
\sum_{t=1}^{H}
\nabla_\theta \log \pi_\theta(a_t \mid s_t)
\left(
\sum_{\tau > t}
H^{\times}\Big(
\pi_T(\cdot \mid s_\tau),
\;\pi_\theta(\cdot \mid s_\tau)
\Big)
\right)
\right]
}_{\text{future-credit assignment}}.
\end{split}
\label{eqn:grad_expression}
\end{equation}
The DistIL gradient decomposes into two complementary terms: a teacher-weighted local imitation term and a future-aware credit-assignment term.
\begin{itemize}
\item The first term provides \emph{teacher-weighted local imitation}. At each visited state $s_t$, its inner gradient is $\nabla_\theta H^{\times}\big(\pi_T(\cdot \mid s_t), \pi_\theta(\cdot \mid s_t)\big) = -\sum_y \pi_T(y \mid s_t)\nabla_\theta \log \pi_\theta(y \mid s_t)$, so the student increases token log-probabilities in proportion to the teacher’s probabilities. This teacher-weighted form is central to DistIL: unlike SDPO-style local updates, which weight student log-probabilities by a student--teacher mismatch $\log \frac{\pi_T(y \mid s_t)}{\pi_\theta(y \mid s_t)}$, DistIL imitates the teacher distribution directly. This direct imitation is what yields the reward-aligned update direction behind DistIL's monotonic-improvement guarantee (c.f. Proposition \ref{prop:monotonic_main}). The same design also allows the gradient to be estimated from teacher samples when only black-box access is available.
\item The second term provides \emph{future-aware credit assignment}. It carries teacher--student disagreement at later prefixes, measured by cross-entropy, back through $\nabla_\theta \log \pi_\theta(y_t \mid s_t)$ to the earlier tokens that made those prefixes likely.
\end{itemize}

\noindent \textbf{DistIL algorithm.}
Algorithm~\ref{alg:distil_practical} provides a practical variant of DistIL, used in our experiments. The algorithm instantiates the distributional imitation objective in~\eqref{eqn:our_objective_final} with the future-aware credit-assignment term in~\eqref{eqn:grad_expression}. In practice, we normalize the cumulative future loss to avoid length bias in long generations \citep{gu2024minillm,rashidinejad2025sail}. We optimize the objective with a PPO-style trust-region update. Following prior self-distillation methods, the teacher is obtained through an exponential moving average copy of the student \citep{hubotter2026reinforcement} or fixed as the initial student parameters \citep{zhao2026self}, conditioned on the rich feedback.

\begin{algorithm}[t]
\caption{Distributional Imitation Learning (DistIL) with Rich Feedback}
\label{alg:distil_practical}
\begin{algorithmic}[1]
\Require Student policy $\pi_{\theta_1}$, teacher policy $\pi_\phi$ with parameters initialized at $\phi:=\theta_1$, prompt 
distribution $\rho$, learning rate $\eta$, number of iterations $N$, group size $G$, update parameter $\beta \in [0,1]$.
\For{$i = 1, \dots, N$}
    \State Sample a prompt $x \sim \rho$
        \For{$j = 1, \dots, G$}
            \State Sample rollout $y_j \sim \pi_{\theta_i}(\cdot \mid x)$ and collect feedback $f_j $.
            \State Set state $s_t \gets (x, y_{j,<t})$. 
            \For{$t = 1, \dots, |y_j|$}
                \State Compute future credit $C^{\mathrm{fut}}_{j,t} \gets \frac{1}{|y_j| - t - 1}
                    \sum_{\tau > t}
                    H^{\times}\!\Bigl(
                    \mathsf{sg}\bigl(\pi_{\phi}(\cdot \mid s_\tau, f_j)\bigr),\;
                    \pi_{\theta_i}(\cdot \mid s_\tau)
                    \Bigr)$
            \EndFor
    \EndFor
    \State Compute gradient \vspace{-2mm}
    \begin{align*}
        \nabla_\theta \mathcal{L}_{\mathrm{DistIL}} \gets\; \frac{1}{G}\sum_{j=1}^G\sum_{t=1}^{|y_j|} C^{\mathrm{fut}}_{j,t} \cdot \nabla_\theta \log \pi_{\theta_i}(y_{j,t} \mid s_t) - 
         \sum_a  \underbrace{\mathsf{sg}\bigl(\pi_{\phi}(a \mid s_t, f_j)\bigr)}_{=\text{Local credit}} \cdot \nabla_\theta \log \pi_{\theta_i}(a \mid s_t)
    \end{align*}
    \State Update $\theta_i$ using PPO-style to get $\theta_{i+1}$
    \State Update teacher parameters: $\phi \gets \beta\phi + (1-\beta)\theta_{i+1}$.
\EndFor
\State \Return $\hat{\pi} = \pi_{\theta_{N+1}}$
\end{algorithmic}
\end{algorithm}

\section{Theoretical Properties of DistIL}
We now turn to the theoretical properties of DistIL. Our analysis shows that the same design choices that make DistIL natural for rich feedback also align it with reward improvement. The forward cross-entropy objective yields a monotonic policy-improvement guarantee; the resulting online imitation-learning procedure admits a sublinear regret bound; and the objective optimizes a teacher-weighted lower bound on the maximum likelihood of success, explaining why DistIL can improve pass@$n$ in practice.

\subsection{DistIL enjoys monotonic policy improvement}

Section~\ref{sec:policy_improvement} showed that distilling a better teacher via $f$-divergences is not enough to guarantee policy improvement and the surrogate objective must also point in a reward-improving direction. We now show that DistIL has precisely this property. Under a fixed state distribution and a mild local realizability condition, the natural-gradient step induced by the forward cross-entropy objective improves reward whenever the teacher is on average better than the student. The proof is deferred to Appendix~\ref{subsec:app_monotonic}.

\begin{proposition}[\textbf{DistIL enjoys monotonic policy improvement}]
\label{prop:monotonic_main} Assume the map $\theta \mapsto \pi_\theta$ is twice continuously differentiable. Let $\Delta$ denote the teacher--student gap as defined in \eqref{eq:teacher_student_gap} and assume teacher--student tangent space realizability: there exists $u \in \mathbb{R}^d$ such that for all states $s$ and actions $y$, $\log \pi_T(a \mid s) - \log \pi_\theta(y \mid s) = u^\top \nabla_\theta \log \pi_\theta(y \mid s).$ Then, under a fixed state distribution, the natural policy gradient step applied to DistIL $\theta' = \theta - \eta \, F(\theta)^{-1} \nabla_{\theta} H^{\times}(\pi_T \| \pi_\theta)$ with sufficiently small $\eta > 0$ satisfies:
\begin{align*}
    J(\pi_{\theta'}) = J(\pi_\theta) + \eta \Delta + O(\eta^2).
\end{align*}
Consequently, if $\Delta > 0$, the update gives monotonic policy improvement.
\end{proposition}

Proposition~\ref{prop:monotonic_main} explains why the forward cross-entropy objective is aligned with reward improvement. Under a fixed state distribution, DistIL's natural-gradient update moves the student policy toward the teacher policy itself. In first order, $\pi_{\theta'}(\cdot \mid s)$ becomes an interpolation between $\pi_\theta(\cdot \mid s)$ and $\pi_T(\cdot \mid s)$. Thus, when the teacher has higher expected reward, small movement toward the teacher improves the updated policy.

Reverse-KL objective in SDPO does not have this property. Its update is weighted by the log-ratio $\log(\pi_\theta/\pi_T)$, so student does not necessarily move toward the teacher. As Section~\ref{sec:policy_improvement} shows, this weighting can suppress actions that are good but overrepresented by the student relative to the teacher, and can even increase probability on worse actions after normalization. This distinction is not related to the usual mode-seeking versus mode-covering behavior of forward and reverse KL. The argument here is on first-order reward alignment: forward cross-entropy preserves the direction toward a better teacher, while reverse KL can distort it.

\textit{Remark.} The tangent-space realizability assumption in Proposition~\ref{prop:monotonic_main} requires the teacher direction to be locally representable by the student policy class. In other words, the mismatch from the student toward the teacher must lie in the span of the student's score functions. This is a mild condition for expressive policy classes and holds exactly for tabular softmax policies; see Appendix~\ref{subsec:app_tangent_space_realizability} for more details.

\begin{AIbox}{Takeaway: DistIL's update preserves the teacher direction}
Under the natural-gradient update, DistIL's forward cross-entropy objective moves the student toward the teacher policy itself. This yields first-order policy improvement whenever the teacher is better on average. Reverse-KL objectives can distort this direction through log-ratio weighting, and therefore may not improve reward even when the teacher is better.
\end{AIbox}

\subsection{DistIL regret guarantee}
We next analyze DistIL as an online learning algorithm. Unlike the monotonic-improvement result, which studies a single local update, the regret analysis accounts for the fact that each student update changes the state distribution on which the teacher is queried. The theorem below shows that the policy returned by DistIL achieves a decaying suboptimality gap in terms of the expected reward relative to the teacher policy. The proof can be found in Appendix \ref{app:regret_proof}.

\begin{theorem}[\textbf{DistIL regret guarantee}]
\label{thm:suboptimality_d_dagger}
Assume finite ratio-based concentrability coefficient $C < \infty$ (Definition~\ref{def:concentrability}) and finite KL concentrability coefficient $C_0 < \infty$ (Definition~\ref{def:kl_concentrability}). Let $\bar{\sigma}_{\pi_T}^2$ denote the worst-case teacher-policy variance (Definition~\ref{def:teacher_variance}) and let $\mu_T$ denote the signed teacher recoverability parameter (Definition~\ref{def:teacher_recoverability}). Then, the policy $\hat \pi$ returned by the NPG-DistIL Algorithm \ref{alg:distil_npg} with learning rate $\eta = \frac{1}{C}\sqrt{\frac{2C_0}{n}}$ satisfies:
\begin{align*}
    J(\pi_T)-J(\hat \pi)
    \;\lesssim\;
    \sqrt{\bar{\sigma}_{\pi_T}^2 \frac{HC}{2}}
    \left(\frac{2C_0}{n}\right)^{1/4}
    +
    \frac{\mu_T HC}{2}
    \sqrt{\frac{2C_0}{n}}.
\end{align*}
\end{theorem}

Theorem~\ref{thm:suboptimality_d_dagger} makes explicit what controls the difficulty of learning from a privileged teacher. The concentrability coefficients $C$ and $C_0$ reflect the student--teacher coverage. The regret bound also contains two quantities that reflect the problem-dependent difficulty of conducting imitation learning in the environment. The variance term $\bar{\sigma}_{\pi_T}^2$ measures the stochasticity of the teacher, i.e., how much the value of the teacher’s sampled actions can fluctuate. The recoverability parameter $\mu_T$ measures the worst-case value loss from taking one arbitrary action at a state and then following the teacher thereafter. When the teacher is nearly deterministic, the second term dominates and yields a faster $O(n^{-1/2})$ rate. In contrast, for highly stochastic teachers, the variance term leads to a slower $O(n^{-1/4})$ rate.

\textit{Remark (Proof technique).} The proof builds on the per-state mirror descent analysis of \citet{chang2024dataset}, but requires a non-trivial extension to the online setting. In particular, their analysis relies on expectations with respect to a fixed distribution, which does not hold in our setting because the state distribution evolves after every student update. We handle this by deriving a recurrence relation for the intermediate error terms and showing that the KL divergence contracts as the student is updated, allowing us to extend the analysis to the online regime. See Appendix~\ref{subsec:app_regret} for further details.

\subsection{DistIL maximizes a teacher-weighted likelihood of success}
\label{subsec:connection_to_maxrl}

In this section, we show that DistIL can also be understood from a sequence-level success perspective. While the previous results explain why the forward cross-entropy objective induces reward-aligned updates and admits monotonic policy improvement and regret guarantees, they do not yet explain why the method is particularly effective in sampling-based test-time metrics such as Pass@$N$ and Majority@$N$ (c.f., Figure \ref{fig:science_complete_metrics}). In the proposition below, we show that when rewards are binary indicators of success, optimizing DistIL's distributional imitation objective also maximizes a lower bound on a teacher-weighted log-likelihood of success.

\begin{proposition}[\textbf{DistIL objective is a lower bound on the teacher-weighted likelihood of success}]
\label{prop:weighted_connection}
For a fixed prompt $x$ and any policy $\pi$, denote the success probability $p_\pi(x) := \mathbb{E}_{y \sim \pi(\cdot \mid x)}[r(x,y)]$, where $r(x,y)\in\{0,1\}$ indicates correctness. Then
\[
- H^{\times}\!\big(\pi_T(\cdot \mid x), \pi_\theta(\cdot \mid x)\big)
\le
p_{\pi_T}(x)\log p_{\pi_\theta}(x).
\]
\end{proposition}

In Proposition~\ref{prop:weighted_connection}, the right-hand side is the student's log probability of success on prompt $x$, weighted by the teacher's probability of producing a successful response on the same prompt. Thus, minimizing forward cross-entropy maximizes a lower bound on a teacher-weighted log-likelihood of success, placing more weight on prompts where the feedback-conditioned teacher is likely to be correct. The proof, given in Appendix~\ref{subsec:distil_max_rl_connection}, follows by decomposing the teacher cross-entropy into correct and incorrect responses and lower bounding the contribution of correct responses, while noting that the remaining terms are nonnegative.

This connection is specific to the teacher-weighted forward cross-entropy structure of DistIL. In contrast, objectives such as SDPO need not yield the same lower bound on the student's probability of success. See Appendix~\ref{subsec:sdpo_max_rl_non_connection} for a formal statement and analysis. Recent work on maximum-likelihood RL shows that maximizing the likelihood of successful trajectories improves Pass@$N$ for all $N$ \citep{tajwar2026maximum}. Proposition~\ref{prop:weighted_connection} therefore provides a principled explanation for DistIL's strong Pass@$N$ performance.

\section{Experimental results}\label{sec:experiments}

We evaluate our method across three representative settings that capture different feedback regimes and levels of supervision: (i) science reasoning benchmarks without rich environment feedback; (ii) coding tasks with rich execution-based feedback; and (iii) mathematical reasoning on challenging problems with access to ground-truth solutions.

\subsection{Bootstrapping on model generated correct responses as feedback}

We evaluate in the standard RLVR setting, where only a binary success signal is observed per rollout. Following SDPO, we sample multiple rollouts per input and use successful trajectories as feedback $f$, training the model to increase their likelihood and converting sparse rewards into a dense, self-supervised signal.

\noindent\textbf{Setup.} We compare against SDPO, on-policy GRPO (trajectories from the current policy), and off-policy GRPO (past rollouts reweighted via importance sampling). For all baselines, we adopt the same hyperparameter settings as in SDPO. Scientific reasoning is assessed on undergraduate-level problems in chemistry, physics, biology, and materials science using the L3 reasoning subset of SciKnowEval~\citep{feng2024sciknoweval}, with a train--test split for in-domain generalization. We initialize from Qwen38B~\citep{yang2025qwen3} and Olmo3-7B-Instruct~\citep{olmo2025olmo},
and report Best@16 and Maj@16 (Figure~\ref{fig:science_complete_metrics}) and Avg@16
(Table~\ref{tab:results_sparse_feedback_science}) against wall-clock training time using the \texttt{verl} framework~\citep{sheng2025hybridflow}.

\begin{table}[h]
  \centering
  \caption{Comparison on scientific reasoning benchmarks (SciKnowEval L3). We report best
  avg$@16$ within 1h and 5h of wall-clock training on $4\times$ NVIDIA H200 GPUs. Average
  is computed over all four domains. \textbf{Bold} denotes the best result per column;
  \underline{underline} denotes the second best.}
  \resizebox{0.8\textwidth}{!}{%
  \begin{tabular}{l cccccccccc}
    \toprule
    & \multicolumn{2}{c}{Chemistry}
    & \multicolumn{2}{c}{Physics}
    & \multicolumn{2}{c}{Biology}
    & \multicolumn{2}{c}{Materials}
    & \multicolumn{2}{c}{Average} \\
    \cmidrule(lr){2-3}\cmidrule(lr){4-5}\cmidrule(lr){6-7}\cmidrule(lr){8-9}\cmidrule(lr){10-11}
    Method & 1h & 5h & 1h & 5h & 1h & 5h & 1h & 5h & 1h & 5h \\
    \midrule
    \textbf{Qwen3-8B}
      & \multicolumn{2}{c}{$41.2$}
      & \multicolumn{2}{c}{$59.2$}
      & \multicolumn{2}{c}{$30.8$}
      & \multicolumn{2}{c}{$58.9$}
      & \multicolumn{2}{c}{$47.5$} \\
    + GRPO (off-policy)
      & $65.9$ & $74.5$
      & $63.8$ & $\underline{72.7}$
      & $35.1$ & $59.9$
      & $\underline{74.3}$ & $\mathbf{77.1}$
      & $59.8$ & $71.1$ \\
    + GRPO (on-policy)
      & $63.3$ & $63.4$
      & $63.6$ & $63.6$
      & $49.8$ & $49.8$
      & $73.9$ & $74.1$
      & $62.7$ & $62.7$ \\
    + SDPO
      & $\underline{73.0}$ & $\underline{80.2}$
      & $\underline{68.0}$ & $72.4$
      & $\underline{52.9}$ & $\underline{63.6}$
      & $72.2$ & $75.9$
      & $\underline{66.5}$ & $\underline{73.0}$ \\
    + DistIL \textit{(Ours)}
      & $\mathbf{75.8}$ & $\mathbf{80.8}$
      & $\mathbf{72.7}$ & $\mathbf{80.8}$
      & $\mathbf{53.3}$ & $\mathbf{66.6}$
      & $\mathbf{74.9}$ & $\underline{76.2}$
      & $\mathbf{69.2}$ & $\mathbf{76.1}$ \\
    \midrule
    \textbf{Olmo3-7B-Instruct}
      & \multicolumn{2}{c}{$22.8$}
      & \multicolumn{2}{c}{$37.7$}
      & \multicolumn{2}{c}{$16.2$}
      & \multicolumn{2}{c}{$36.7$}
      & \multicolumn{2}{c}{$28.4$} \\
    + GRPO (off-policy)
      & $39.7$ & $56.7$
      & $55.3$ & $63.3$
      & $35.6$ & $\mathbf{55.8}$
      & $70.9$ & $75.0$
      & $50.4$ & $62.7$ \\
    + GRPO (on-policy)
      & $51.4$ & $57.5$
      & $\underline{62.7}$ & $62.7$
      & $\mathbf{49.8}$ & $49.8$
      & $\underline{73.3}$ & $73.5$
      & $59.3$ & $60.9$ \\
    + SDPO
      & $\underline{70.2}$ & $\underline{79.2}$
      & $59.8$ & $\underline{64.9}$
      & $\underline{49.5}$ & $52.9$
      & $71.8$ & $\mathbf{78.1}$
      & $\underline{62.8}$ & $\underline{68.8}$ \\
    + DistIL \textit{(Ours)}
      & $\mathbf{72.1}$ & $\mathbf{81.0}$
      & $\mathbf{67.4}$ & $\mathbf{74.5}$
      & $47.8$ & $\underline{55.3}$
      & $\mathbf{73.5}$ & $\underline{76.9}$
      & $\mathbf{65.2}$ & $\mathbf{71.9}$ \\
    \bottomrule
  \end{tabular}%
  }
  \label{tab:results_sparse_feedback_science}
\end{table}

\noindent\textbf{Empirical Findings.}
Table~\ref{tab:results_sparse_feedback_science} and Figure~\ref{fig:science_complete_metrics} summarize our results. DistIL achieves the best avg@16 performance in the large majority of columns across both model families, with the largest gains in physics and chemistry domains; physics gains reach $8.1$/$9.6$ points at $5$h over next-best on Qwen3-8B/OLMo3. Beyond accuracy, DistIL's lead emerges within ${\sim}20$ steps and is largely monotonically sustained (Proposition~\ref{prop:monotonic_main}), while SDPO exhibits instability: Best@16 collapses on biology after step $100$ and oscillates on chemistry and physics. DistIL's Best@16 advantage is grounded in Proposition~\ref{prop:weighted_connection}: our forward cross-entropy objective maximizes a teacher-weighted lower bound on success probability, implicitly optimizing \texttt{Pass@N} for all N \citep{tajwar2026maximum}. On-policy GRPO saturates within one hour; off-policy GRPO lags
DistIL by $5{+}$ points on reasoning-heavy domains at $5$h.

\subsection{Learning with rich environment feedback}

We evaluate in environments with rich feedback, where signals go beyond binary success, such as execution errors, which standard RLVR cannot exploit. Following SDPO, we sample multiple rollouts per input and use execution-based feedback to construct supervision. Concretely, rollouts are evaluated against public unit tests, and feedback $f$ is derived from the resulting code execution logs.

\noindent \textbf{Setup.} We compare against SDPO and GRPO with identical hyperparameters as in SDPO on LiveCodeBench (LCBv6)~\citep{jain2024livecodebench}, using public tests for training feedback and private tests for evaluation, initialized from Qwen3-8B~\citep{yang2025qwen3}. Following \citep{jain2024livecodebench}, we report exact correctness by determining whether all test cases are passed (Score/Avg@$k$, Score/Maj@$k$, Score/Best@$k$) and partial correctness: fraction of test cases passed (Accuracy/Avg@$k$, Accuracy/Maj@$k$, Accuracy/Best@$k$), at temperature $0.2$, top\_p $0.95$.

\noindent\textbf{Empirical Findings.}
DistIL achieves Accuracy/Avg@16$\,{=}\,0.656$ and Score/Avg@16$\,{=}\,0.482$, outperforming
SDPO ($0.643$, $0.467$) and GRPO ($0.600$, $0.405$). The large gap over GRPO reflects its
fundamental inability to exploit execution feedback. Figure~\ref{fig:lcb_eval_temp_0.2} further shows consistent gains across all $k$ on both Best@$k$ and Maj@$k$ for score and accuracy, with the advantage most pronounced at small $k$. The gap between DistIL and SDPO could be attributed to our theoretically grounded-objective and squence-level credit assignment.

\begin{figure}[h]
    \centering
    \includegraphics[width=0.24\linewidth]{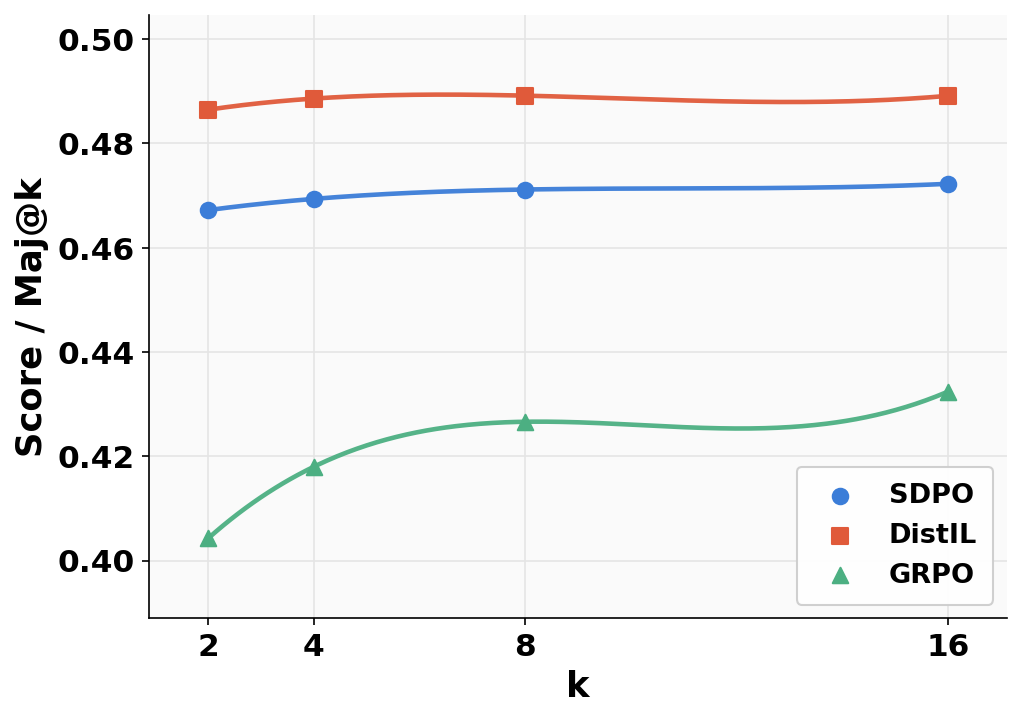}
    \includegraphics[width=0.24\linewidth]{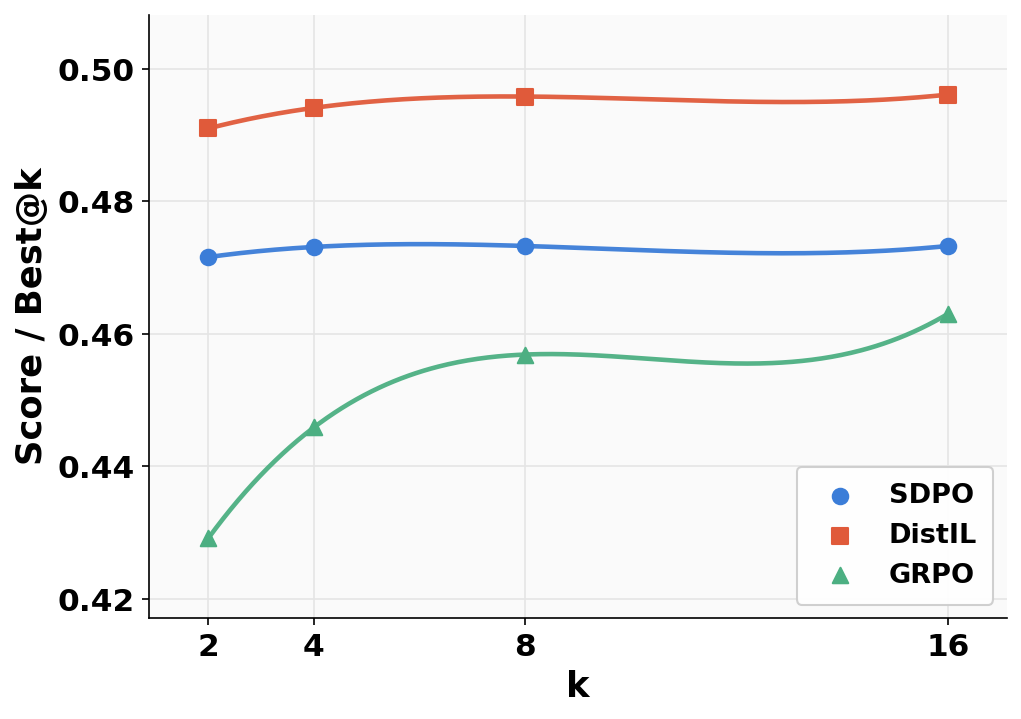}
    \includegraphics[width=0.24\linewidth]{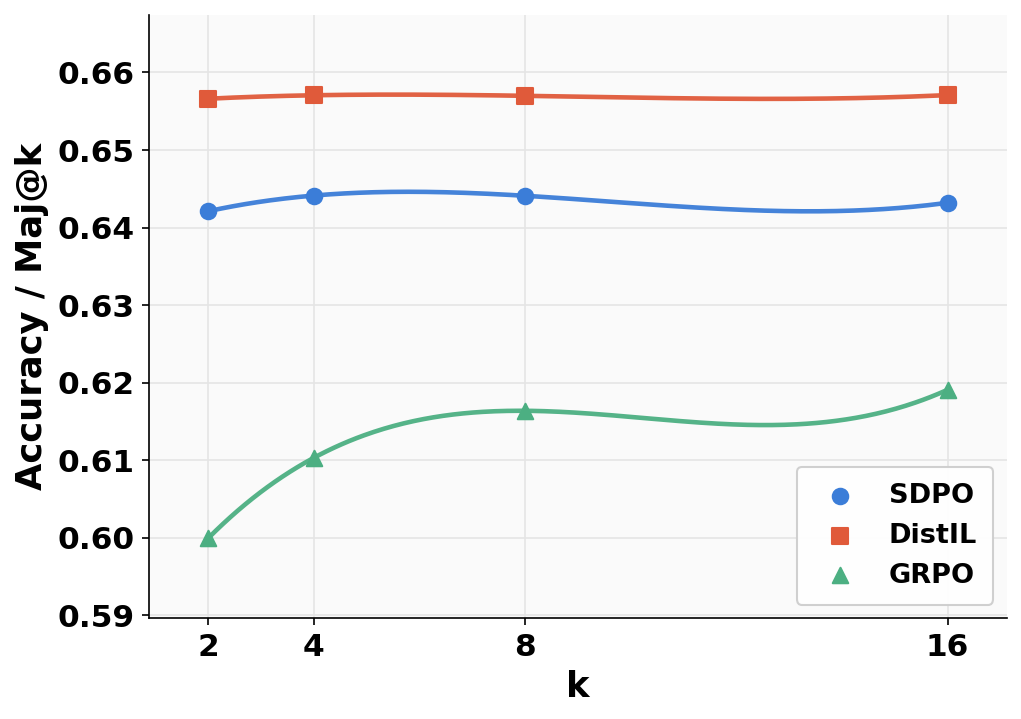}
    \includegraphics[width=0.24\linewidth]{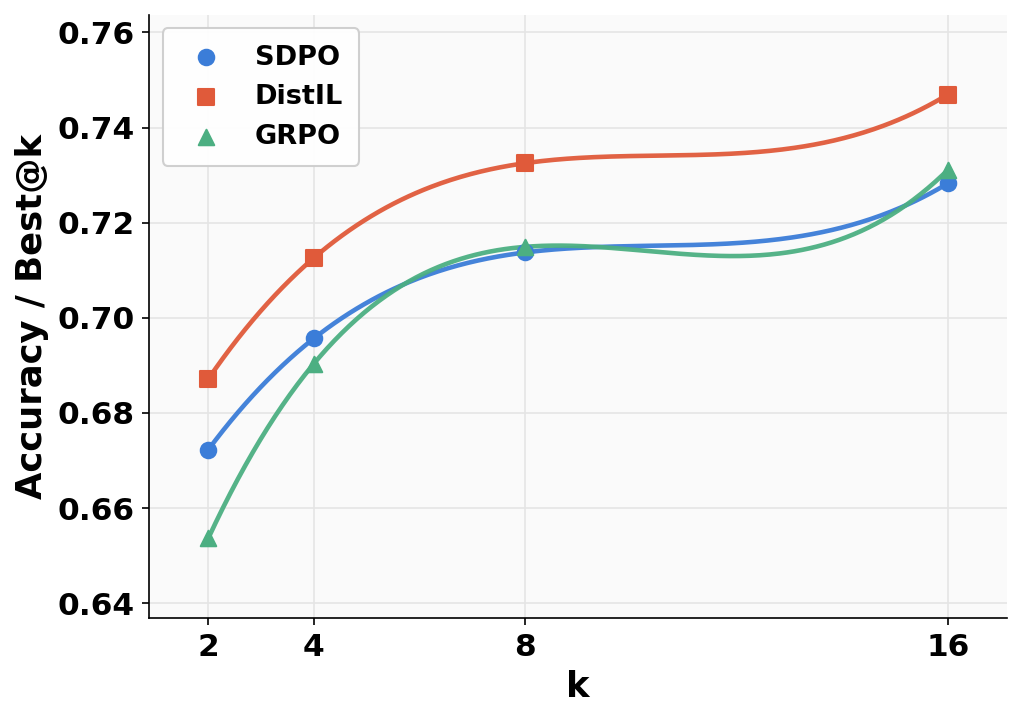}
    \caption{LCBv6 evaluation at $\tau{=}0.2$ for checkpoint at step-80 (following SDPO), reporting Score and Accuracy at Best@$k$ and Maj@$k$ for $k \in \{2,4,8,16\}$.}
    \label{fig:lcb_eval_temp_0.2}
\end{figure}

\subsection{Learning to solve hard mathematical reasoning problems}

We evaluate with access to ground-truth solutions as feedback $f$ on very hard mathematical reasoning problems, enabling learning on problems where the base model never produces a correct rollout. In this regime, RLVR methods such as GRPO receive zero advantage and cannot improve, while distillation-based approaches directly leverage ground-truth solutions to make progress.

\noindent\textbf{Setup.} We construct a challenging training set of $738$ hard mathematical problems sourced from \citep{cmu_aire_pope_hard_oracle_2026} and OmniMath (problems with pass@512$\,{=}\,0$ for Qwen3-4B-Instruct), following \citep{setlur2026reuse,qu2026pope}. We compare against OPSD (forward-KL), SDPO (reverse-KL), GRPO, and SFT, initialized from Qwen3-8B and Qwen3-4B-Instruct-2507 models. Evaluation covers AIME24, AIME25, HMMT25, AMC23, and Minerva.

\noindent\textbf{Empirical Findings.} Table~\ref{tab:qwen3_math_results} summarizes results. GRPO exactly matches the base model on both scales, confirming the expected failure: zero pass rate yields zero advantage. SFT degrades below the base model on most benchmarks, indicating overfitting rather than generalizable reasoning. DistIL achieves the best results across the large majority of columns at both model scales. Gains are most substantial on AIME25, where DistIL leads the next-best method on Avg@16 by $3.8$ points on Qwen3-4B and $1.4$ points on
Qwen3-8B. On the stronger Qwen3-8B, where SDPO and OPSD are competitive, DistIL
consistently ranks first or second, with no degradation relative to the base model.

\begin{table}[h]
  \centering
  \caption{Performance comparison on mathematical reasoning benchmarks for Qwen3 models. Sampling configuration is given in Appendix \ref{app:eval_hyper_math}. \textbf{Bold} denotes the best result per column;
           \underline{underline} denotes the second best. Following OPSD \citep{zhao2026self}, we report the best checkpoint up to step $100$.}
  \resizebox{0.8\textwidth}{!}{%
  \begin{tabular}{ll cccccccccc}
    \toprule
    & & \multicolumn{2}{c}{AIME24}
      & \multicolumn{2}{c}{AIME25}
      & \multicolumn{2}{c}{HMMT25}
      & \multicolumn{2}{c}{AMC23}
      & \multicolumn{2}{c}{Minerva} \\
    \cmidrule(lr){3-4}\cmidrule(lr){5-6}\cmidrule(lr){7-8}\cmidrule(lr){9-10}\cmidrule(lr){11-12}
    Model & Method & Avg & Pass & Avg & Pass & Avg & Pass & Avg & Pass & Avg & Pass \\
    \midrule
    \multirow{6}{*}{Qwen3-4B}
      & Base
        & $61.4$ & $82.0$ & $50.3$ & $69.5$ & $30.3$ & $48.8$ & $93.8$ & $98.8$  & $43.2$ & $48.6$ \\
      & $+$ SFT
        & $55.8$ & $80.9$ & $43.1$ & $67.4$ & $29.7$ & $46.0$ & $91.7$ & $97.0$  & $41.9$ & $50.4$ \\
      & $+$ GRPO
        & $61.4$ & $82.0$ & $50.3$ & $69.5$ & $30.3$ & $48.8$ & $93.8$ & $98.8$  & $43.2$ & $48.6$ \\
      & $+$ SDPO
        & $60.9$ & $\underline{85.4}$ & $49.6$ & $74.1$ & $\underline{32.8}$ & $50.9$ & $93.9$ & $\underline{99.8}$ & $\underline{43.4}$ & $\underline{52.0}$ \\
      & $+$ OPSD
        & $\underline{63.2}$ & $\underline{85.4}$ & $\underline{51.5}$ & $\underline{76.1}$ & $\mathbf{33.0}$ & $\underline{54.2}$ & $\underline{94.5}$ & $99.6$ & $\underline{43.4}$ & $51.8$ \\
      & $+$ DistIL \textit{(Ours)}
        & $\mathbf{65.3}$ & $\mathbf{87.5}$ & $\mathbf{55.3}$ & $\mathbf{77.6}$ & $\mathbf{33.0}$ & $\mathbf{56.6}$ & $\mathbf{94.8}$ & $\mathbf{100.0}$ & $\mathbf{44.2}$ & $\mathbf{52.9}$ \\
    \midrule
    \multirow{6}{*}{Qwen3-8B}
      & Base
        & $75.9$ & $86.8$ & $66.9$ & $79.6$ & $44.7$ & $65.6$ & $95.6$ & $\mathbf{100.0}$ & $48.9$ & $56.3$ \\
      & $+$ SFT
        & $74.2$ & $85.3$ & $65.3$ & $80.7$ & $42.2$ & $64.4$ & $96.3$ & $\mathbf{100.0}$ & $48.2$ & $56.2$ \\
      & $+$ GRPO
        & $75.9$ & $86.8$ & $66.9$ & $79.6$ & $44.7$ & $65.6$ & $95.6$ & $\mathbf{100.0}$ & $48.9$ & $56.3$ \\
      & $+$ SDPO
        & $\mathbf{76.9}$ & $89.5$ & $68.3$ & $\underline{84.0}$ & $\underline{45.6}$ & $\underline{71.2}$ & $\underline{96.5}$ & $\mathbf{100.0}$ & $\underline{49.2}$ & $57.7$ \\
      & $+$ OPSD
        & $\underline{76.5}$ & $\mathbf{91.3}$ & $\underline{69.7}$ & $83.3$ & $45.5$ & $68.2$ & $96.2$ & $\mathbf{100.0}$ & $\underline{49.2}$ & $\underline{57.8}$ \\
      & $+$ DistIL \textit{(Ours)}
        & $76.4$ & $\underline{90.7}$ & $\mathbf{71.1}$ & $\mathbf{85.0}$ & $\mathbf{46.4}$ & $\mathbf{71.4}$ & $\mathbf{96.6}$ & $\mathbf{100.0}$ & $\mathbf{49.5}$ & $\mathbf{58.4}$ \\
    \bottomrule
  \end{tabular}%
  }
  \label{tab:qwen3_math_results}
\end{table}

\begin{greenAIbox}{Result: DistIL outperforms RLVR and self-distillation baselines.}
On scientific reasoning, DistIL achieves the best Avg@16 on both model families, with gains up to 9.6 points. On coding, DistIL leads on all Best@$k$ and Maj@$k$ metrics; GRPO's large gap reflects its inability to exploit execution feedback. On hard mathematical reasoning problems, where GRPO receives zero gradients and fails to improve the base model, DistIL surpasses on-policy distillation baselines in most settings, with gains up to 3.8 points on AIME25 Avg@16. These results verify the advantages of reward alignment (Proposition~\ref{prop:monotonic_main}) and future-aware credit assignment of DistIL for learning from rich feedback.
\end{greenAIbox}

\subsection{Ablation Study}
\label{app:ablation}

\subsubsection{Credit Assignment}
We compare DistIL against a CE (cross-entropy) baseline that retains only the cross-entropy term of the gradient, performing local credit assignment analogous to SDPO and OPSD. Figure~\ref{fig:credt_comparison} shows results on the Material domain. DistIL consistently outperforms CE throughout training, demonstrating that the full credit assignment across the entire response distribution rather than assigning it locally, is the key driver of performance gains. The CE baseline exhibits higher variance and fails to improve to the level of DistIL's performance, suggesting that local credit assignment is insufficient for learning robust scientific reasoning.

\begin{AIbox}{Takeaway: Future credit assignment is advantageous in DistIL.}
Replacing full sequence-level credit assignment with local token-wise gradients (CE 
baseline) leads to higher variance and strictly worse performance, confirming 
Proposition~\ref{prop:grad_issue}.
\end{AIbox}

\begin{figure}[t]
    \centering
    \begin{minipage}[t]{0.49\textwidth}
        \centering
        \includegraphics[width=0.8\textwidth]{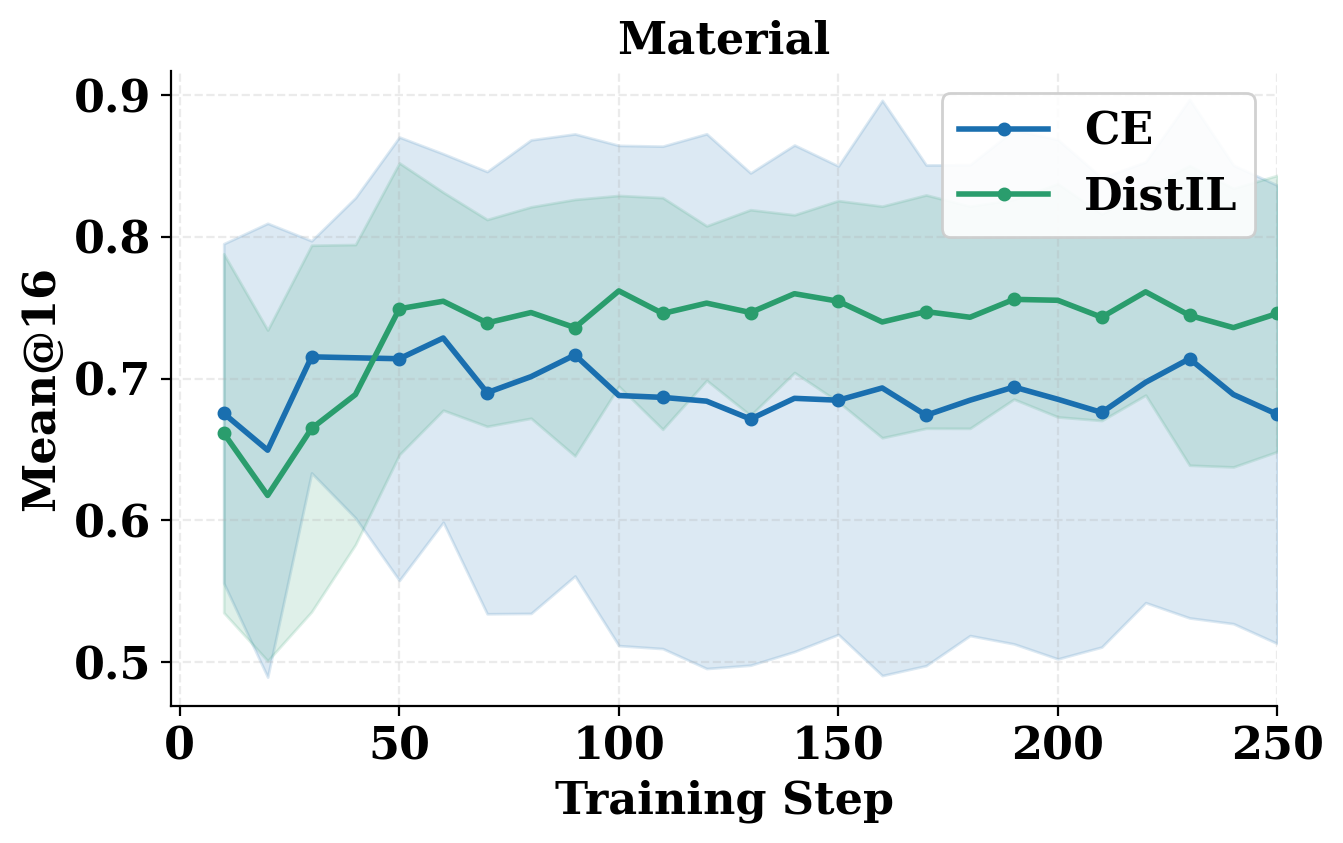}
        \caption{Comparison between DistIL (Ours) and CE. The difference between these methods lies in the credit assignment. CE only does local token-wise creadit assignment similar to SDPO and OPSD, whereas DistIL does full credit assignment as given in Eq. \eqref{eqn:grad_expression}.}
        \label{fig:credt_comparison}
    \end{minipage}
    \hfill
    \begin{minipage}[t]{0.49\textwidth}
        \centering
        \includegraphics[width=0.8\textwidth]{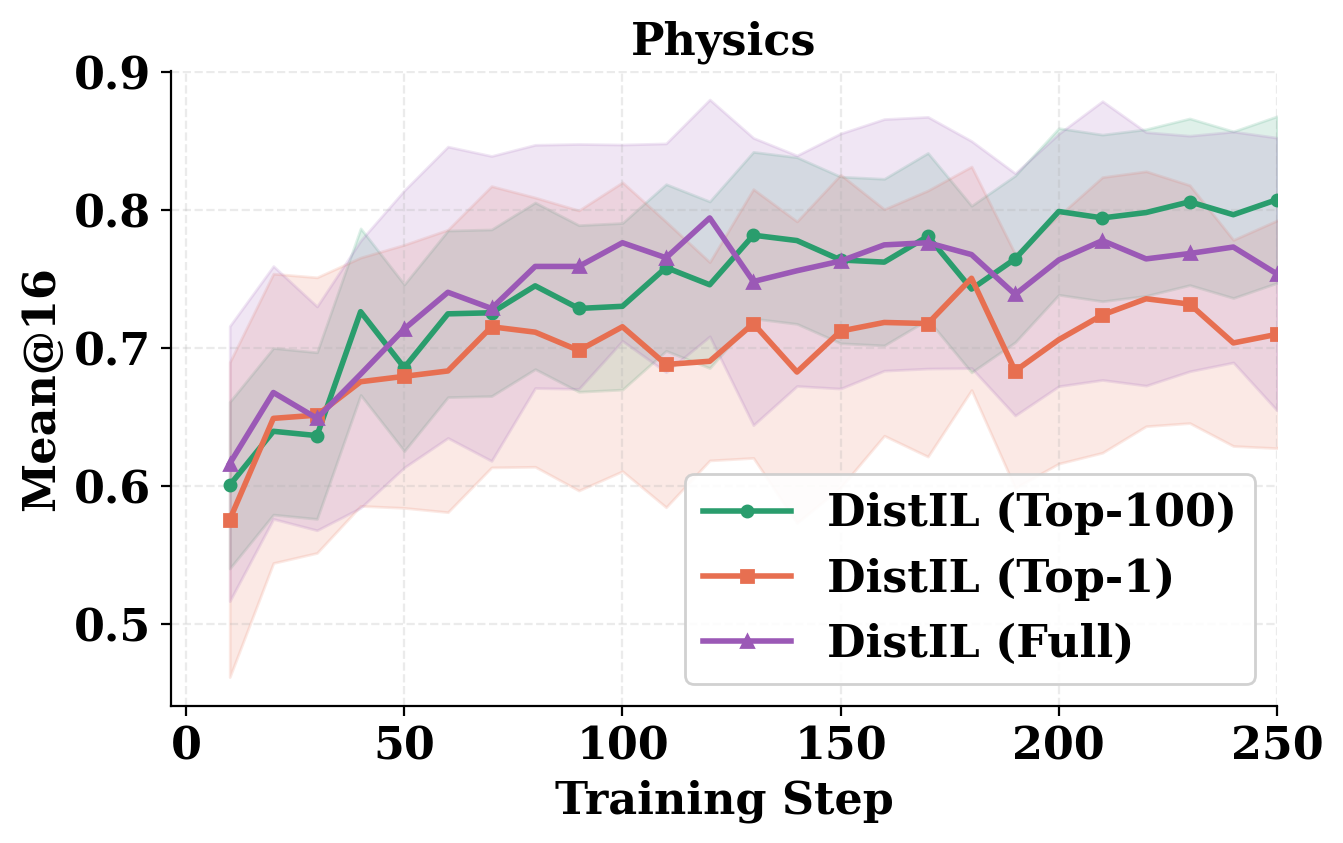}
        \caption{omparison of Avg@16 between different number of top-k tokens used for training DistIL on Material benchmark.}
        \label{fig:topk_ablation}
    \end{minipage}
\end{figure}

\subsubsection{Top-1 v/s Top k v/s Full distribution}
 We ablate the number of teacher tokens used in the distillation objective across $K \in \{1, 100, \text{Full}\}$ on the Physics domain (Figure~\ref{fig:topk_ablation}). Top-1 distillation, which aligns the student only to the single highest-probability teacher token, underperforms throughout training, reflecting insufficient signal from a single token. Top-100 achieves the best final performance and most stable training trajectory, outperforming both Top-1 and Full distillation. Full distillation, while competitive, introduces noise from low-probability tokens that are unlikely to be informative, slightly degrading performance relative to Top-100. Based on this analysis, we use K=100 in all experiments.
\section{Related Work}
\label{sec:related_work}

\textbf{Credit Assignment in RLVR.}
RLVR enables training from automatically verifiable outcomes without human preference annotations \citep{shao2024deepseekmath, guo2025deepseek}. However, supervision is applied at the trajectory level, assigning identical signals to all tokens despite their differing roles, making credit assignment especially challenging for long-horizon reasoning. One line of work addresses this by introducing intermediate supervision via process-reward models or step-wise evaluators \citep{lightman2023let, zhang2025lessons, wang2024math, yang2026test, luo2024improve, chen2024step, zhang2024generative, yang2025dynamic, dai2025s, cui2025process}. While effective, these approaches require additional annotation or auxiliary models. An alternative direction avoids explicit supervision by reweighting updates using internal signals such as entropy or uncertainty, but these rely on heuristic notions of token importance \citep{xie2025unlocking, li2026outcome, cheng2026reasoning, wang2025beyond, chen2025seed, sun2025ktae, li2025attention, chen2025beyond}. In our work, instead of estimating token importance heuristically or learning separate reward models, we construct token-level supervision directly through privileged feedback. This yields structured credit assignment while preserving the simplicity of the verifier-only training pipeline.

\textbf{On-Policy Distillation and On-Policy Self-Distillation.} On-policy distillation trains a student on its own rollouts while a teacher provides token-level supervision via divergence-based objectives \citep{agarwal2024policy, xu2024speculative, lu2025onpolicydistillation, xiao2026mimo, yang2025qwen3}. Aligning training with the student’s visitation distribution reduces the distribution mismatch inherent in off-policy approaches, but it typically relies on a stronger external teacher. Recent work extends this paradigm to self-distillation, where a single model provides both supervision and updates. OPSD \citep{zhao2026self} formulates this as minimizing the divergence between a privileged teacher, which is conditioned on ground-truth reasoning traces, and a student, thereby improving token efficiency. SDPO \citep{hubotter2026reinforcement} generalizes this setup to incorporate arbitrary feedback signals (e.g., code execution error logs), while RLSD \citep{yang2026self} integrates self-distillation with the GRPO methodologically to stabilize training. However, self-distillation introduces new challenges. It can degrade reasoning by suppressing epistemic uncertainty \citep{kim2026does}, and its success depends critically on alignment between teacher and student distributions as well as the quality of self-generated targets \citep{li2026rethinking}. Recent methods address these issues through improved supervision and routing strategies, including competence-aware weighting (PACED), iterative self-revision (SD-Zero), and hybrid routing between distillation and RL updates (SRPO) \citep{xu2026paced, he2026self, li2026unifying}. In contrast, our approach revisits the objective underlying self-distillation. We adopt a cross-entropy-based formulation with well-behaved optimization properties, resulting in more stable training dynamics and improved gradient estimation. This is related to methods such as DAgger \citep{ross2011reduction} in imitation learning that query the expert for supervision on the learner’s induced state distribution.
\section{Conclusion}
\label{sec:conclusion}

This work studies reinforcement learning from rich feedback. We connect learning from rich feedback via distilling a privileged feedback-conditioned teacher, to the standard expected-reward maximization objective in RL. This perspective raises a natural question: if the feedback-conditioned privileged teacher is better than the student on average, can we design a self-distillation algorithm whose updates monotonically improve the student toward the better teacher and admit regret guarantees? We show that existing self-distillation approaches can fail in two distinct ways: their objectives do not guarantee monotonic reward improvement and their local tokenwise gradient approximation can miss the delayed consequences of early decisions. In response, we introduce DistIL, a distributional variant of DAgger that optimizes a forward cross-entropy objective and performs future-aware credit assignment. We prove that DistIL enjoys monotonic policy improvement, admits a regret guarantee, and implicitly optimizes a teacher-weighted lower bound on the probability of success, explaining improvements in Pass@$N$ for every $N$. Empirically, we show that DistIL achieves strong performance across scientific reasoning, coding, and solving hard mathematical reasoning problems.

\bibliographystyle{assets/plainnat}
\bibliography{paper}

\clearpage

\newpage 

\appendix
\appendixtoc

\newpage

\appsection{Proofs for theoretical limitations of prior work}\label{sec:appendix_prior}

\appsubsection{\texorpdfstring{Proof of Proposition \ref{prop:f_divergence_non_monotonic}: $f$-divergence self-distillation does not guarantee monotonic policy improvement}{Proof of Proposition \ref{prop:f_divergence_non_monotonic}: f-divergence self-distillation does not guarantee monotonic policy improvement}}\label{subsec:app_monotonic_f_div}

Let us define $
p_a := \pi_\theta(a)$, $q_a := \pi_T(a)$, and $u_a := \frac{q_a}{p_a}$. Then, we have
\[
\mathcal{L}(\theta) := D_f(\pi_T\|\pi_\theta)
= \sum_{a\in\mathcal A} p_a\, f(u_a).
\]
Next, we compute the gradient of $\mathcal{L}(\theta)$. By the chain rule,
\[
\frac{\partial L(\theta)}{\partial p_a}
=
\frac{\partial}{\partial p_a}
\left[
p_a f\!\left(\frac{q_a}{p_a}\right)
\right]
=
f(u_a) - u_a f'(u_a)=g(a),
\]
where we used the definition of $g(a)$ from the statement in the Proposition. Next, we define the score vector $s_a := \nabla_\theta \log \pi_\theta(a)$. Using $\nabla_\theta p_a = p_a s_a$ and linearity, we have
\[
\nabla_\theta \mathcal{L}(\theta)
=
\sum_{b\in\mathcal A} \frac{\partial \mathcal{L}(\theta)}{\partial p_b}\,\nabla_\theta p_b
=
\sum_{b\in\mathcal A} g(b)\,p_b\,s_b
=
\mathbb{E}_{b\sim\pi_\theta}\!\left[g(b)s_b\right].
\]

Now, the natural gradient update is given as $
\theta' - \theta
=
-\eta F(\theta)^{-1}\nabla_\theta \mathcal{L}(\theta)$. Upon expanding the policy map to first order gives
\[
\pi_{\theta'}(a)
=
\pi_\theta(a)
+
\nabla_\theta \pi_\theta(a)^\top(\theta'-\theta)
+
O(\eta^2).
\]
Since $\nabla_\theta \pi_\theta(a)=\pi_\theta(a)s_a$ and using the gradient update, we get
\[
\pi_{\theta'}(a)
=
\pi_\theta(a)
-
\eta\,\pi_\theta(a)\,
s_a^\top F(\theta)^{-1}\nabla_\theta \mathcal{L}(\theta)
+
O(\eta^2).
\]
Now, we substitute the expression for $\nabla_\theta \mathcal{L}(\theta)$ to get $s_a^\top F(\theta)^{-1}\nabla_\theta \mathcal{L}(\theta)
=
\mathbb{E}_{b\sim\pi_\theta}
\!\left[
g(b)\,s_a^\top F(\theta)^{-1} s_b
\right]$. By Lemma~\ref{lem:fisher_identity},
\[
s_a^\top F(\theta)^{-1} s_b
=
\frac{\mathbf{1}\{a=b\}}{\pi_\theta(b)} - 1.
\]
Therefore,
\begin{align*}
s_a^\top F(\theta)^{-1}\nabla_\theta \mathcal{L}(\theta)
&=
\mathbb{E}_{b\sim\pi_\theta}
\left[
g(b)\left(\frac{\mathbf{1}\{a=b\}}{\pi_\theta(b)} - 1\right)
\right] \\
&=
\mathbb{E}_{b\sim\pi_\theta}
\left[
g(b)\frac{\mathbf{1}\{a=b\}}{\pi_\theta(b)}
\right]
-
\mathbb{E}_{b\sim\pi_\theta}[g(b)] \\
&=
g(a) - \mathbb{E}_{b\sim\pi_\theta}[g(b)].
\end{align*}
Thus, we have
\[
\pi_{\theta'}(a)
=
\pi_\theta(a)
-
\eta\,\pi_\theta(a)
\Big(
g(a)-\mathbb{E}_{b\sim\pi_\theta}[g(b)]
\Big)
+
O(\eta^2)
.
\]
Next, we consider that change in value of the policy following gradient update. Upon using the definition of $J(\pi)$, we have
\begin{align*}
J(\pi_{\theta'}) - J(\pi_\theta)
&=
\sum_{a\in\mathcal A} r(a)\big(\pi_{\theta'}(a)-\pi_\theta(a)\big) \\
&=
-\eta\sum_{a\in\mathcal A}\pi_\theta(a)\,r(a)
\Big(
g(a)-\mathbb{E}_{b\sim\pi_\theta}[g(b)]
\Big)
+
O(\eta^2) \\
&=
-\eta\left(
\mathbb{E}_{a\sim\pi_\theta}[r(a)g(a)]
-
\mathbb{E}_{a\sim\pi_\theta}[r(a)]\,
\mathbb{E}_{a\sim\pi_\theta}[g(a)]
\right)
+
O(\eta^2) \\
&=
-\eta\,\Cov_{a\sim\pi_\theta}\!\big(r(a),g(a)\big)
+
O(\eta^2).
\end{align*}
Hence, we have $J(\pi_{\theta'}) - J(\pi_\theta) = -\eta\,\Cov_{y\sim\pi_\theta}\!\big(r(y),g(y)\big)
+
O(\eta^2).$

\appsubsection{Proof of Proposition~\ref{prop:reverse_kl_non_improvement_counterexample}: Reverse-KL distillation can decrease reward}
\label{subsec:reverse_kl_failure_example}

We consider a single-context three-armed bandit with rewards
\[
r(y_1) = 1, \quad r(y_2) = \tfrac{1}{2}, \quad r(y_3) = 0,
\]
and student and teacher policies
\[
\pi_\theta = \!\left(\tfrac{1}{20},\, \tfrac{9}{20},\, \tfrac{1}{2}\right),
\qquad
\pi_T = \!\left(\tfrac{3}{10},\, \tfrac{1}{20},\, \tfrac{13}{20}\right).
\]
A direct calculation gives
\[
J(\pi_\theta) = \tfrac{1}{20}\cdot 1 + \tfrac{9}{20}\cdot\tfrac{1}{2} 
= \tfrac{11}{40} = 0.275,
\qquad
J(\pi_T) = \tfrac{3}{10}\cdot 1 + \tfrac{1}{20}\cdot\tfrac{1}{2} 
= \tfrac{13}{40} = 0.325,
\]
so $\Delta = J(\pi_T) - J(\pi_\theta) = \tfrac{1}{20} > 0$.

The bad action probability increases: The reverse-KL NPG update takes the form
\[
\pi_\theta'(y)
=
\pi_\theta(y)
-
\eta\,\pi_\theta(y)
\left(
\log\frac{\pi_\theta(y)}{\pi_T(y)}
-
D_{\mathrm{KL}}(\pi_\theta\|\pi_T)
\right).
\]
We compute the individual log-ratios:
\[
\log\frac{\pi_\theta(y_1)}{\pi_T(y_1)} = \log\tfrac{1}{6} \approx -1.792,
\quad
\log\frac{\pi_\theta(y_2)}{\pi_T(y_2)} = \log 9 \approx 2.197,
\quad
\log\frac{\pi_\theta(y_3)}{\pi_T(y_3)} = \log\tfrac{10}{13} \approx -0.262,
\]
and the KL divergence:
\[
D_{\mathrm{KL}}(\pi_\theta\|\pi_T)
=
\tfrac{1}{20}\log\tfrac{1}{6}
+
\tfrac{9}{20}\log 9
+
\tfrac{1}{2}\log\tfrac{10}{13}
\approx 0.768.
\]
For the bad action $y_3$, the update coefficient is
\[
\log\frac{\pi_\theta(y_3)}{\pi_T(y_3)} - D_{\mathrm{KL}}(\pi_\theta\|\pi_T)
\approx -0.262 - 0.768 = -1.030 < 0,
\]
so $\pi_\theta'(y_3) = 0.50 + \eta \cdot 0.515 > \pi_\theta(y_3)$. Therefore, the probability of the suboptimal action increases after the update

From Proposition~\ref{prop:f_divergence_non_monotonic}, the first-order 
change in expected reward is $J(\pi_\theta') - J(\pi_\theta) = 
-\eta\,\mathrm{Cov}_{y\sim\pi_\theta}(r(y),\log(\pi_\theta(y)/\pi_T(y))) 
+ O(\eta^2)$. We compute:
\[
\mathrm{Cov}_{y\sim\pi_\theta}\!\left(r(y),\log\frac{\pi_\theta(y)}{\pi_T(y)}\right)
=
\mathbb{E}_{\pi_\theta}\!\left[r\log\frac{\pi_\theta}{\pi_T}\right]
-
\mathbb{E}_{\pi_\theta}[r]\cdot\mathbb{E}_{\pi_\theta}\!\left[\log\frac{\pi_\theta}{\pi_T}\right]
\approx 0.4048 - 0.275\times 0.768 \approx 0.1936 > 0.
\]
Therefore, for all sufficiently small $\eta > 0$,
\[
J(\pi_\theta') \leq J(\pi_\theta).
\]
Hence, even when $\Delta > 0$, reverse KL does not guarantee monotonic improvement.

\appsubsection{Proof of Proposition \ref{prop:grad_issue}: Local credit assignment can be strictly suboptimal}\label{subsec:app_local_gradient_issue}

Let $\mathcal A=\{a,b,c\}$ and consider a two-step decision process $(y_1,y_2)\in\mathcal A^2$.
We define the student policy $\pi_{u,v}$ by
\begin{align*}
\pi_{u,v}(y_1=\cdot)&=\mathrm{softmax}(u,0,0),\\
\pi_{u,v}(y_2=\cdot\mid y_1=a)&=\mathrm{softmax}(v,0,0),\\
\pi_{u,v}(y_2=\cdot\mid y_1=b)&=U,\\
\pi_{u,v}(y_2=\cdot\mid y_1=c)&=U,
\end{align*}
where $U$ is the uniform distribution over actions. Let the teacher policy $\pi_T$ be
\[
\pi_T(y_1)=U, \qquad
\pi_T(\cdot\mid y_1=a)=R, \qquad
\pi_T(\cdot\mid y_1=b)=\pi_T(\cdot\mid y_1=c)=U,
\]
where
\[
R=\left(\frac14,\frac{3+2\sqrt2}{8},\frac{3-2\sqrt2}{8}\right).
\]
Define the sequence-level objective
\[
\mathcal L_{\mathrm{seq}}(u,v)=\KL\bigl(\pi_{u,v}(y_1,y_2)\,\|\,\pi_T(y_1,y_2)\bigr),
\]
and let $\mathcal L_{\mathrm{curr}}$ denote the tokenwise local approximation that sums
per-step KL terms but treats the prefix distribution as fixed when differentiating
(as in SDPO/OPSD).

Define the reward by
\[
r_1(x,y_1)=\mathbf 1\{y_1=b\}, \qquad r_2(x,y_{1:2})=0.
\]
Since the context is deterministic, we suppress $x$ below.

Write
\[
p(u):=\Pr(y_1=a)=\frac{e^u}{e^u+2},
\qquad
\Pr(y_1=b)=\Pr(y_1=c)=\frac{1-p(u)}{2},
\]
and
\[
q(v):=\Pr(y_2=a\mid y_1=a)=\frac{e^v}{e^v+2}.
\]

By the chain rule for KL divergence,
\[
\mathcal L_{\mathrm{seq}}(u,v)
=
\KL\bigl(\pi_{u,v}(y_1)\|U\bigr)
+
\Pr(y_1=a)\,\KL\bigl(\pi_{u,v}(y_2\mid y_1=a)\|R\bigr).
\]
The first term equals
\[
K(p(u)), \qquad
K(p)=p\log(3p)+(1-p)\log\!\left(\frac{3(1-p)}{2}\right).
\]

Now define
\begin{align*}
D(v)
&:=\KL\bigl(\pi_{u,v}(y_2\mid y_1=a)\|R\bigr)\\
&=
q(v)\log\frac{q(v)}{1/4}
+\frac{1-q(v)}{2}\log\frac{(1-q(v))/2}{(3+2\sqrt2)/8}
+\frac{1-q(v)}{2}\log\frac{(1-q(v))/2}{(3-2\sqrt2)/8}.
\end{align*}
Since
\[
\left(\frac{3+2\sqrt2}{8}\right)\left(\frac{3-2\sqrt2}{8}\right)=\frac{1}{64},
\]
the last two terms combine to
\[
(1-q(v))\log\frac{(1-q(v))/2}{1/8}.
\]
Hence
\[
D(v)
=
q(v)\log(4q(v))+(1-q(v))\log\!\bigl(4(1-q(v))\bigr).
\]
Differentiating with respect to $q$ gives
\[
\frac{dD(v)}{dq}=\log\frac{q}{1-q},
\]
so the unique minimizer is $q=\tfrac12$, i.e. $v=\log 2$, and at this point
\[
\inf_v D(v)=\log 2.
\]

Substituting this into the sequence-level loss yields the effective one-dimensional objective
\[
F(p)=K(p)+p\log 2.
\]
Now
\[
K'(p)=\log\frac{2p}{1-p},
\]
so
\[
F'(p)=K'(p)+\log 2=\log\frac{2p}{1-p}+\log 2=\log\frac{4p}{1-p}.
\]
Also,
\[
F''(p)=\frac1p+\frac1{1-p}>0,
\]
so $F$ is strictly convex and has a unique minimizer at
\[
F'(p)=0
\quad\Longleftrightarrow\quad
\frac{4p}{1-p}=1
\quad\Longleftrightarrow\quad
p=\frac15.
\]
Since $p(u)=e^u/(e^u+2)$, this corresponds to $u=-\log 2$.

We now compare the local and full gradients at the initialization $(u,v)=(0,0)$.
Since
\[
p(0)=q(0)=\frac13,
\qquad
K'\!\left(\frac13\right)=0,
\]
the local approximation gives
\[
\frac{\partial \mathcal L_{\mathrm{curr}}}{\partial u}(u,v)
=
\frac{d}{du}K(p(u)),
\]
and therefore
\[
\frac{\partial \mathcal L_{\mathrm{curr}}}{\partial u}(0,0)=0.
\]
Hence local updates do not move $u$, so the resulting policy satisfies $p=\frac13$ and
\[
J(\pi_{\mathrm{local}})
=
\Pr(y_1=b)
=
\frac{1-p}{2}
=
\frac13.
\]

For the full sequence-level objective,
\[
\mathcal L_{\mathrm{seq}}(u,v)=K(p(u))+p(u)D(v),
\]
so
\[
\frac{\partial \mathcal L_{\mathrm{seq}}}{\partial u}(u,v)
=
p'(u)\bigl(K'(p(u))+D(v)\bigr),
\qquad
p'(u)=\frac{2e^u}{(e^u+2)^2}.
\]
At $(u,v)=(0,0)$, the term $K'(p(0))$ vanishes and $D(0)>0$, so
\[
\frac{\partial \mathcal L_{\mathrm{seq}}}{\partial u}(0,0)
=
p'(0)D(0)
=
\frac{2}{9}D(0)
>
0.
\]
Thus the full sequence-level gradient gives a strict update signal on the first-step parameter, while the local approximation does not.

Finally, since the unique minimizer of $F$ is at $p=\frac15$, the full sequence-level optimization converges to a policy with
\[
J(\pi_{\mathrm{seq}})
=
\Pr(y_1=b)
=
\frac{1-\frac15}{2}
=
\frac25.
\]
Therefore,
\[
J(\pi_{\mathrm{seq}})=\frac25>\frac13=J(\pi_{\mathrm{local}}),
\]
which proves the claim in Proposition \ref{prop:grad_issue}.

\appsection{Proofs for theoretical properties of DistIL}\label{sec:appendix_distil}

\appsubsection{Derivation for gradients of the DistIL objective \texorpdfstring{\eqref{eqn:grad_expression}}{(gradient expression)}}\label{subsec:app_grad_derivation}

We start by defining the per-step loss
\[
\ell_t(y,\theta)
:=
H^{\times}\Big(
\mathsf{sg}\big(\pi_{\theta}(\cdot \mid x, y_{1:t-1}, f)\big),
\;\pi_\theta(\cdot \mid x, y_{1:t-1})
\Big),
\]
so that 
\[
\mathcal{L}_{\mathrm{DistIL}}(\theta)
=
\mathbb{E}_{y \sim \pi_\theta}
\left[
\sum_{t=1}^{H} \ell_t(y,\theta)
\right].
\]

By Lemma~\ref{lem:trajectory_score_identity}, applied to the trajectory distribution $y \sim \pi_\theta$, we have
\begin{equation}
\nabla_\theta \mathcal{L}_{\mathrm{DistIL}}(\theta)
=
\mathbb{E}_{y \sim \pi_\theta}
\left[
\sum_{t=1}^{H}\nabla_\theta \ell_t(y,\theta)
\right]
+
\mathbb{E}_{\tau \sim \pi_\theta}
\left[
\left(\sum_{t=1}^{H} \ell_t(\tau,\theta)\right)
\nabla_\theta \log \pi_\theta(\tau)
\right].
\label{eqn:intermediate_grad_loss}
\end{equation}
Since the policy is autoregressive, we have 
$\log \pi_\theta(\tau)=\sum_{t=1}^{H}\log \pi_\theta(a_t\mid s_t)$ and $\nabla_\theta \log \pi_\theta(\tau)
=
\sum_{t=1}^{H}\nabla_\theta \log \pi_\theta(a_t\mid s_t)$. Upon using this in Eq. \ref{eqn:intermediate_grad_loss} gives
\[
\nabla_\theta \mathcal{L}_{\mathrm{DistIL}}(\theta)
=
\mathbb{E}_{\tau \sim \pi_\theta}
\left[
\sum_{t=1}^{H}\nabla_\theta \ell_t(\tau,\theta)
\right]
+
\mathbb{E}_{\tau \sim \pi_\theta}
\left[
\sum_{t=1}^{H}\sum_{i=1}^{H}
\ell_i(\tau,\theta)\,\nabla_\theta \log \pi_\theta(a_t\mid s_t)
\right].
\]

We now show that only future terms $i>t$ survive in the second expectation. Let us fix $t$ and consider any $i\le t$. The quantity $\ell_i(\tau,\theta)$ depends only on the prefix up to time $i$, and therefore is independent of the sampled action $a_t$ once the state $s_t$ is fixed. Consequently,
\[
\mathbb{E}_{a_t\sim \pi_\theta(\cdot\mid s_t)}
\Big[
\nabla_\theta \log \pi_\theta(a_t\mid s_t)\,\ell_i(\tau,\theta)
\Big]
=
\ell_i(\tau,\theta)\,
\mathbb{E}_{a_t\sim \pi_\theta(\cdot\mid s_t)}
\Big[
\nabla_\theta \log \pi_\theta(a_t\mid s_t)
\Big].
\]
By Lemma~\ref{lem:trajectory_score_identity},
\[
\mathbb{E}_{a_t\sim \pi_\theta(\cdot\mid s_t)}
\Big[
\nabla_\theta \log \pi_\theta(a_t\mid s_t)
\Big]
=0,
\]
so the contribution from all $i\le t$ vanishes. Therefore,
\[
\mathbb{E}_{\tau \sim \pi_\theta}
\left[
\left(\sum_{t=1}^{H} \ell_t(\tau,\theta)\right)
\nabla_\theta \log \pi_\theta(\tau)
\right]
=
\mathbb{E}_{\tau \sim \pi_\theta}
\left[
\sum_{t=1}^{H}
\nabla_\theta \log \pi_\theta(a_t\mid s_t)
\left(
\sum_{i>t}\ell_i(\tau,\theta)
\right)
\right].
\]
Substituting back the definition of $\ell_i$ gives
\begin{equation*}
\begin{split}
\nabla_\theta \mathcal{L}_{\mathrm{DistIL}}
= 
& \mathbb{E}_{\tau \sim \pi_\theta}
\left[
\sum_{t=1}^{H}
\nabla_\theta 
H^{\times}\Big(
\mathsf{sg}\big(\pi_{\theta}(\cdot \mid s_t, f)\big),
\;\pi_\theta(\cdot \mid s_t)
\Big)
\right]\\
&+
\mathbb{E}_{\tau \sim \pi_\theta}
\left[
\sum_{t=1}^{H}
\nabla_\theta \log \pi_\theta(a_t \mid s_t)
\left(
\sum_{i > t}
H^{\times}\Big(
\mathsf{sg}\big(\pi_{\theta}(\cdot \mid s_i, f)\big),
\;\pi_\theta(\cdot \mid s_i)
\Big)
\right)
\right].
\end{split}
\end{equation*}

\appsubsection{Soft-max tabular classes satisfy tangent-space realizability}\label{subsec:app_tangent_space_realizability}
\begin{proposition}
\label{prop:tangent_softmax}
Let $\mathcal{S}$ and $\mathcal{A}$ be finite sets, and consider a tabular softmax policy
\[
\pi_\theta(a \mid s)
=
\frac{\exp(\theta(s,a))}{\sum_{a' \in \mathcal{A}} \exp(\theta(s,a'))}.
\]
Then, for any fixed $\theta$ and any teacher policy $\pi_T$, there exists a vector
$u : \mathcal{S} \times \mathcal{A} \to \mathbb{R}$ such that for all
$s \in \mathcal{S}$ and $a \in \mathcal{A}$,
\[
\pi_T(a \mid s) - \pi_\theta(a \mid s)
=
\pi_\theta(a \mid s)\,
\nabla_\theta \log \pi_\theta(a \mid s)^\top u.
\]
\end{proposition}

\begin{proof}
We start by fixing $s \in \mathcal{S}$. Since $\pi_\theta(\cdot \mid s)$ is a softmax distribution,
$\pi_\theta(a \mid s) > 0$ for all $a \in \mathcal{A}$, so the ratio
$\pi_T(a \mid s) / \pi_\theta(a \mid s)$ is well-defined. For tabular softmax, the score function satisfies
\[
\frac{\partial \log \pi_\theta(a \mid s)}{\partial \theta(s,a')}
=
\mathbf{1}\{a'=a\} - \pi_\theta(a' \mid s).
\]
Thus, for any vector $u$, we have
\[
\nabla_\theta \log \pi_\theta(a \mid s)^\top u
=
u(s,a) - \sum_{a' \in \mathcal{A}} \pi_\theta(a' \mid s)\,u(s,a').
\]

Now, let us define
\[
u(s,a) := \frac{\pi_T(a \mid s)}{\pi_\theta(a \mid s)}.
\]
Then, we have
\[
\sum_{a'} \pi_\theta(a' \mid s)\,u(s,a')
=
\sum_{a'} \pi_T(a' \mid s)
=
1.
\]
Hence,
\[
\nabla_\theta \log \pi_\theta(a \mid s)^\top u
=
\frac{\pi_T(a \mid s)}{\pi_\theta(a \mid s)} - 1.
\]
Multiplying both sides by $\pi_\theta(a \mid s)$ gives
\[
\pi_\theta(a \mid s)\,\nabla_\theta \log \pi_\theta(a \mid s)^\top u
=
\pi_T(a \mid s) - \pi_\theta(a \mid s).
\]
\end{proof}

\appsubsection{Proof of Proposition \ref{prop:monotonic_main}: DistIL enjoys monotonic policy improvement}\label{subsec:app_monotonic}

Under the frozen state distribution, the gradient of $\mathcal{L}_{\mathrm{DistIL}}$ simplifies to
\[
\nabla_\theta \mathcal{L}_{\mathrm{DistIL}}(\theta)
=
- \mathbb{E}_{x \sim \rho,\; y \sim \pi_T(\cdot \mid x)}
\left[
\nabla_\theta \log \pi_\theta(y \mid x)
\right].
\]

From Lemma \ref{lem:trajectory_score_identity}, $
\mathbb{E}_{y \sim \pi_\theta(\cdot \mid x)}
[\nabla_\theta \log \pi_\theta(y \mid x)] = 0$. Therefore, the gradient can also be written as
\begin{align*}
    \nabla_\theta \mathcal{L}_{\mathrm{DistIL}}(\theta) = - \mathbb{E}_{x \sim \rho} \left[\sum_y \left( \pi_T(y \mid x) - \pi_\theta(y \mid x) \right) \nabla \log \pi_\theta(y \mid x)\right].
\end{align*}

Next, we apply the tangent assumption $
\pi_T(y \mid x) - \pi_\theta(y \mid x)
=
\pi_\theta(y \mid x)\,\nabla_\theta \log \pi_\theta(y \mid x)^\top u$  to obtain
\[
\nabla_\theta \mathcal{L}_{\mathrm{DistIL}}(\theta)
=
- \mathbb{E}_{x \sim \rho}
\left[
\sum_y
\pi_\theta(y \mid x)
\nabla_\theta \log \pi_\theta(y \mid x)
\nabla_\theta \log \pi_\theta(y \mid x)^\top
\right] u.
\]

We now recognize the Fisher information matrix, 
\[F(\theta)
=
\mathbb{E}_{x \sim \rho,\; y \sim \pi_\theta(\cdot \mid x)}
\left[
\nabla_\theta \log \pi_\theta(y \mid x)
\nabla_\theta \log \pi_\theta(y \mid x)^\top
\right]
\]
to conclude
\[
\nabla_\theta \mathcal{L}_{\mathrm{DistIL}}(\theta) = -F(\theta) u.
\]

Thus, the natural gradient update becomes
\[
\theta' = \theta + \eta u.
\]

We now analyze the change in the policy using a first-order Taylor expansion
\[
\pi_{\theta'}(y \mid x)
=
\pi_\theta(y \mid x)
+
\nabla_\theta \pi_\theta(y \mid x)^\top (\theta' - \theta)
+ O(\eta^2).
\]

Substituting $\theta' - \theta = \eta u$ gives
\[
\pi_{\theta'}(y \mid x)
=
\pi_\theta(y \mid x)
+
\eta \nabla_\theta \pi_\theta(y \mid x)^\top u
+ O(\eta^2).
\]

Using $
\nabla_\theta \pi_\theta(y \mid x)
=
\pi_\theta(y \mid x)\nabla_\theta \log \pi_\theta(y \mid x)$, we obtain
\[
\pi_{\theta'}(y \mid x)
=
\pi_\theta(y \mid x)
+
\eta \pi_\theta(y \mid x)\nabla_\theta \log \pi_\theta(y \mid x)^\top u
+ O(\eta^2).
\]

Next, we again apply the tangent assumption to obtain
\[
\pi_{\theta'}(y \mid x)
=
\pi_\theta(y \mid x)
+
\eta \big(\pi_T(y \mid x) - \pi_\theta(y \mid x)\big)
+ O(\eta^2).
\]

Multiplying by $r(x,y)$, summing over $y$, and taking expectation over $x \sim \rho$, we obtain
\[
\mathbb{E}_{x,y \sim \pi_{\theta'}}[r(x,y)]
=
\mathbb{E}_{x,y \sim \pi_\theta}[r(x,y)]
+
\eta
\left(
\mathbb{E}_{x,y \sim \pi_T}[r(x,y)]
-
\mathbb{E}_{x,y \sim \pi_\theta}[r(x,y)]
\right)
+ O(\eta^2).
\]

Finally, using the definition of $\Delta$ and $J(\pi)$, this becomes
\[
J(\pi_{\theta'})
=
J(\pi_\theta)
+ \eta \Delta + O(\eta^2),
\]

Since $\Delta > 0$, for sufficiently small $\eta > 0$, the update gives strict improvement.

\appsubsection{DistIL regret analysis}\label{subsec:app_regret}

This section proves the regret guarantee for DistIL. We begin by stating a natural-policy-gradient variant of DistIL in Algorithm \ref{alg:distil_npg}. We then provide definitions for quantities that appear in the analysis: ratio-based and KL-based concentrability (coverage) coefficients between the teacher and student policies \citep{rashidinejad2021bridging, chang2024dataset} as well as two imitation-learning quantities that measure the difficulty of matching the teacher, namely teacher-policy variance and teacher recoverability parameter \citep{foster2024is}. With these definitions in place, Appendix~\ref{app:regret_proof} proves the regret bound for DistIL.

\subsubsection{Natural-policy gradient variant of DistIL}\label{app:NPG_distil_algorithm}

NPG-DistIL instantiates DistIL with a natural policy gradient update rule. Rather than  applying a single global gradient step, the algorithm performs a \emph{per-state} mirror descent update following \citep{chang2024dataset}, where the mirror map is chosen to be the negative entropy, yielding the KL divergence as the associated Bregman divergence. Concretely, at each round $i$, the algorithm rolls out trajectories under the current student policy $\pi_{\theta_i}$, collects the set of visited state-action prefixes, and for each visited state $s$ minimizes a linearization of the statewise cross-entropy loss $\ell_s$ subject to a KL proximity penalty to the previous iterate $\pi_{\theta_i}(\cdot \mid s)$. The full procedure is detailed in Algorithm~\ref{alg:distil_npg}.

\begin{algorithm}[h]
\caption{NPG-DistIL}
\label{alg:distil_npg}
\begin{algorithmic}[1]
\Require Initial student policy $\pi_{\theta_1}$, teacher policy $\pi_T$, prompt distribution $\rho$, step size $\eta$, number of rounds $n$
\For{$i=1,\dots,n$}
    \State Sample a prompt batch $\mathcal{B}_i \sim \rho$
    \State Roll out trajectories $\mathcal{Y}_i(x) \sim \pi_{\theta_i}(\cdot \mid x)$ for each $x \in \mathcal{B}_i$
    \State Collect the visited prefixes $\mathcal{S}_i
        \gets
        \left\{
        (x, y_{1:t-1})
        \;:\;
        x \in \mathcal{B}_i,\;
        y \in \mathcal{Y}_i(x),\;
        t=1,\dots,|y|
        \right\}.$
    \For{each $s \in \mathcal{S}_i$}
        \State Define $\ell_s(\pi_{\theta_i}) \gets -\sum_{a\in\mathcal A}\pi_T(a\mid s)\log \pi_{\theta_i}(a\mid s)$
        \State Update the local policy by mirror descent
        \begin{equation}
            \pi_{\theta_{i+1}}(\cdot\mid s)
        \gets
        \arg\min_{q\in\Delta(\mathcal A)}
        \left\{
        \left\langle \nabla \ell_s(\pi_{\theta_i}), q \right\rangle
        +
        \frac{1}{\eta}D_{\mathrm{KL}}\!\left(q\|\pi_{\theta_i}(\cdot\mid s)\right)
        \right\}
        \label{eqn:mirror_descent_policy_update}
        \end{equation}
    \EndFor
\EndFor
\State \Return $\hat{\pi} = \pi_{\theta_{n+1}}$
\end{algorithmic}
\end{algorithm}

\subsubsection{Concentrability, variance, and recoverability coefficients}\label{app:regret_defs}

We begin by defining the ratio-based and KL-based concentrability coefficients used in our analysis, which are standard definitions in reinforcement learning \citep{rashidinejad2021bridging, chang2024dataset}.

\begin{definition}[\textbf{Ratio-based concentrability coefficient}]
\label{def:concentrability}
Let $\{\pi_{\theta_i}\}_{i \geq 1}$ denote the sequence of student policies generated by DistIL, and let $\pi_T$ denote the teacher policy. The concentrability coefficient is
\begin{align*}
    C := \sup_{i \geq 1} \sup_{s \in \mathcal{S},\, a \in \mathcal{A}}
    \frac{\pi_T(a \mid s)}{\pi_{\theta_i}(a \mid s)}.
\end{align*}
\end{definition}

\begin{definition}[\textbf{Initial teacher--student KL concentrability coefficient}]
\label{def:kl_concentrability}
Let $\pi_{\theta_1}$ denote the initial student policy and let $\pi_T$ denote the teacher policy. The initial KL concentrability coefficient is
\begin{align*}
C_0 := \sup_{s \in \mathcal{S}}
D_{\mathrm{KL}}\big(\pi_T(\cdot \mid s) || \pi_{\theta_1}(\cdot \mid s)\big).
\end{align*}
\end{definition}

Assuming finite concentrability coefficients $C, C_0 < \infty$ is natural in our setting. The teacher is typically a feedback-conditioned variant of the student policy, often obtained from the initial student or from an exponential-moving-average copy of the student \citep{hubotter2026reinforcement}. Thus, the teacher and student are expected to have substantial support overlap at the start of training. Moreover, natural policy gradient and trust-region methods such as PPO constrain successive policy changes and help preserve this overlap across iterates.

We next introduce two problem-dependent quantities, following \citet{foster2024is}, that capture the difficulty of imitating a stochastic teacher.

\begin{definition}[\textbf{Worst-case teacher variance}]
\label{def:teacher_variance}
For any policy $\pi$, let $P^{\pi \circ_h \pi_T}$ denote the trajectory distribution that follows $\pi$ up to time $h-1$ and then follows $\pi_T$ from time $h$ onward. Let $Q_h^{\pi_T}(s,a)$ be the teacher reward-to-go after taking action $a$ at state $s$ and then following $\pi_T$, and let $V_h^{\pi_T}(s)=\mathbb{E}_{a \sim \pi_T(\cdot \mid s)}[Q_h^{\pi_T}(s,a)]$. The teacher variance of $\pi_T$ under prefixes induced by $\pi$ is
\begin{align*}
\bar{\sigma}_{\pi_T \mid \pi}^2
:=
\sum_{h=1}^{H}
\mathbb{E}_{(s_h,a_h) \sim P^{\pi \circ_h \pi_T}}
\left[
\left(
Q_h^{\pi_T}(s_h,a_h)
-
V_h^{\pi_T}(s_h)
\right)^2
\right].
\end{align*}
The worst-case teacher variance is $\bar{\sigma}_{\pi_T}^2 := \sup_{\pi} \bar{\sigma}_{\pi_T \mid \pi}^2$.
\end{definition}

\begin{definition}[\textbf{Signed teacher recoverability parameter}]
\label{def:teacher_recoverability}
The teacher recoverability coefficient is
\begin{align*}
\mu_T
:=
\sup_{h \in [H]} \sup_{s \in \mathcal{S},, a \in \mathcal{A}}
\left|
V_h^{\pi_T}(s)
-
Q_h^{\pi_T}(s,a)
\right|.
\end{align*}
\end{definition}

The quantity $\bar{\sigma}_{\pi_T}^2$ characterizes the stochasticity of the teacher policy, and in particular, $\bar{\sigma}_{\pi_T}^2 = 0$ when $\pi_T$ is deterministic. The constant $\mu_T$ measures recoverability: the largest value loss caused by taking an arbitrary action at a state and then returning to the teacher thereafter. Thus, small $\mu_T$ corresponds to forgiving environments where deviations from the teacher can be corrected, while large $\mu_T$ captures settings where a single early mistake can have substantial long-term consequences. The recoverability parameter satisfies $0 \le \mu_T \le R$ when rewards are bounded by $R$. Together, these quantities describe the difficulty of imitation: the learner must cope both with stochastic teacher guidance and with the cost of deviating from it. Moreover, the two quantities are related via the bound $\bar{\sigma}_{\pi_T}^2 \le \mu_T^2 H$, which implies that the stochasticity of the teacher is controlled by the recoverability of the environment.

\subsubsection{Proof of Theorem \ref{thm:suboptimality_d_dagger}: DistIL regret guarantee}\label{app:regret_proof}

We prove the regret bound in three steps. First, we establish an auxiliary result (Theorem~\ref{thm:cross_entropy_loss_difference}) below, which bounds the averaged cross-entropy suboptimality of the NPG-DistIL iterates via a mirror descent analysis. Next, we convert this cross-entropy bound into an averaged Hellinger error using the standard KL-to-Hellinger inequality. Finally, we invoke Proposition~C.1 of \citep{foster2024is} to lift the Hellinger error into a policy suboptimality bound in terms of the teacher variance and recoverability parameters. We begin by stating the key auxiliary result and provide its detailed analysis in Appendix~\ref{app:cross_entropy_loss_difference_proof}.

\begin{theorem}[\textbf{Averaged cross-entropy suboptimality of NPG-DistIL}]
\label{thm:cross_entropy_loss_difference}
Assume finite ratio-based concentrability coefficient $C < \infty$ (Definition~\ref{def:concentrability}) and finite KL concentrability coefficient $C_0 < \infty$ (Definition~\ref{def:kl_concentrability}). Let $\{\pi_{\theta_i}\}_{i=1}^{n+1}$ be the iterates generated by Eq. \ref{eqn:mirror_descent_policy_update} in NPG-DistIL (Algorithm~\ref{alg:distil_npg}). Defining $\ell_i(\pi) = \sum_{t=1}^H \mathbb{E}_{s_t \sim d_t^{\pi_{\theta_i}}}[\ell_{s_t}(\pi)]$, where $\ell_s(\pi)$ is as given in Algorithm~\ref{alg:distil_npg}, the following bound holds upon choosing $
\eta = \sqrt{\frac{2C_0}{C^2 n}}$:
\[
\frac{1}{n}\sum_{i=1}^n \bigl(\ell_i(\pi_{\theta_i})-\ell_i(\pi_T)\bigr) \;\le\; HC\sqrt{\frac{2C_0}{n}}.
\]
\end{theorem}

 We now prove Theorem~\ref{thm:suboptimality_d_dagger} using the suboptimality established in Theorem~\ref{thm:cross_entropy_loss_difference}. For a fixed state $s_t$, using the definition of $\ell_{s_t}(\pi)$ from Algorithm \ref{alg:distil_npg}, a direct computation gives
\[
\ell_{s_t}(\pi) - \ell_{s_t}(\pi_T)
=
D_{\mathrm{KL}}\!\bigl(\pi_T(\cdot \mid s_t) \,\|\, \pi(\cdot \mid s_t)\bigr).
\]
Summing over the horizon, taking expectations over $s_t \sim d_t^{\pi_{\theta_i}}$,
and averaging over $i = 1, \ldots, n$ gives
\[
\epsilon_n
:=
\frac{1}{n}\sum_{i=1}^n \bigl(\ell_i(\pi_{\theta_i}) - \ell_i(\pi_T)\bigr)
=
\frac{1}{n}\sum_{i=1}^n \sum_{t=1}^H
\mathbb{E}_{s_t \sim d_t^{\pi_{\theta_i}}}\!\Bigl[
D_{\mathrm{KL}}\!\bigl(\pi_T(\cdot \mid s_t) \,\|\, \pi_{\theta_i}(\cdot \mid s_t)\bigr)
\Bigr].
\]

Next, for any pair of distributions $(p, q)$, the standard inequality $D_{\mathrm{H}}^2(p,q) \le \frac{1}{2} D_{\mathrm{KL}}(p \| q)$ holds. Applying this point-wise with $p = \pi_T(\cdot \mid s_t)$ and
$q = \pi_{\theta_i}(\cdot \mid s_t)$, then again taking expectations, summing over $t$,
and averaging over $i$ gives
\[
\mathcal{E}_n
:=
\frac{1}{n}\sum_{i=1}^n \sum_{t=1}^H
\mathbb{E}_{s_t \sim d_t^{\pi_{\theta_i}}}\!\Bigl[
D_{\mathrm{H}}^2\!\bigl(\pi_T(\cdot \mid s_t),\, \pi_{\theta_i}(\cdot \mid s_t)\bigr)
\Bigr]
\;\le\;
\frac{\epsilon_n}{2}.
\]

Now, we apply Proposition~C.1 of \citet{foster2024is}, which is valid in our online
setting and bounds regret by a variance term and a recoverability
term
\[
\begin{split}
J(\pi_T) - J(\hat{\pi})
&\;\lesssim\;
\sqrt{\bar{\sigma}_{\pi_T}^2\,\mathcal{E}_n}
+
\mu_T\,\mathcal{E}_n \\
&\;\leq\;
\sqrt{\frac{\bar{\sigma}_{\pi_T}^2}{2}\,\epsilon_n}
+
\frac{\mu_T}{2}\,\epsilon_n,
\qquad
\left(\because\; \mathcal{E}_n \le \frac{\epsilon_n}{2}\right).
\end{split}
\]

Finally, substituting the bound $\epsilon_n \le HC\sqrt{2C_0/n}$ from
Theorem~\ref{thm:cross_entropy_loss_difference} with
$\eta = \sqrt{2C_0/(C^2 n)}$ gives
\[
J(\pi_T) - J(\hat{\pi})
\;\lesssim\;
\sqrt{\bar{\sigma}_{\pi_T}^2\frac{HC}{2}}
\left(\frac{2C_0}{n}\right)^{1/4}
+
\frac{\mu_T HC}{2}\sqrt{\frac{2C_0}{n}}.
\]

\subsubsection{Proof of Theorem \ref{thm:cross_entropy_loss_difference}: Averaged cross-entropy suboptimality of NPG-DistIL}\label{app:cross_entropy_loss_difference_proof}
Let us consider the loss difference
\[
\ell_{s}(\pi)-\ell_{s}(\pi_T)
=
-\sum_{a\in\mathcal A}\pi_T(a\mid s)\log \pi(a\mid s)
+
\sum_{a\in\mathcal A}\pi_T(a\mid s)\log \pi_T(a\mid s) =
D_{\mathrm{KL}}\!\big(\pi_T(\cdot\mid s)\,\|\,\pi(\cdot\mid s)\big).
\]
Thus, 
\begin{equation}
\ell_i(\pi)-\ell_i(\pi_T)
=\sum_{t=1}^H \mathbb{E}_{s_t \sim d_t^{\pi_{\theta_i}}}
\Big[
D_{\mathrm{KL}}\!\big(\pi_T(\cdot\mid s_t)\,\|\,\pi(\cdot\mid s_t)\big)
\Big].
\label{eqn:loss_diff}
\end{equation}
We use shorthand $\pi_{\theta_i}(s) := \pi_{\theta_i}(\cdot \mid s)$ and $D_i(s)
:=
D_{\mathrm{KL}}\!\big(\pi_T(\cdot\mid s)\,\|\,\pi_{\theta_i}(\cdot\mid s)\big)$.
Let us now consider the mirror map to be the negative entropy, so that the associated Bregman divergence is the KL divergence. We run the state-conditioned mirror descent update:
\[
\pi_{\theta_{i+1}}(s)
\in
\arg\min_{q\in\Delta(\mathcal A)}
\left\{
\langle \nabla \ell_{s}(\pi_{\theta_i}), q \rangle
+
\frac{1}{\eta} D_{\mathrm{KL}}(q \| \pi_{\theta_i}(s))
\right\}.
\]
From Lemma \ref{lemma:mirror_descent_update}, we have
\[
\langle \eta \nabla \ell_{s}(\pi_{\theta_i}) + \nabla \Phi(\pi_{\theta_{i+1}}(s)) - \nabla \Phi(\pi_{\theta_i}(s)),\, \pi_T(s) - \pi_{\theta_{i+1}}(s) \rangle \ge 0.
\]
Upon rearranging, we have
\[
\eta \langle \nabla \ell_{s}(\pi_{\theta_i}), \pi_{\theta_{i+1}}(s) - \pi_T(s) \rangle
\le
\langle \nabla \Phi(\pi_{\theta_{i+1}}(s)) - \nabla \Phi(\pi_{\theta_i}(s)),\, \pi_T(s) - \pi_{\theta_{i+1}}(s) \rangle.
\]
Next, we use Lemma \ref{lemma:bregman_three_point} to obtain
\begin{equation}
\begin{aligned}
\eta \langle \nabla \ell_{s}(\pi_{\theta_i}), \pi_{\theta_{i+1}}(s) - \pi_T(s) \rangle
&\le
D_i(s)
-
D_{i+1}(s)
-
D_{\mathrm{KL}}(\pi_{\theta_{i+1}}(s)\|\pi_{\theta_i}(s)).
\end{aligned}
\label{eqn:three_point_new}
\end{equation}
Now, from assumption, we have,
\[
\left|
\frac{\partial \ell_{s}(\pi_{\theta_i})}{\partial \pi(a\mid s)}\right| = \left|\frac{\pi_T(a|s)}{\pi_{\theta_i}(a|s)}\right|\leq C.
\]
Next, we define
\[
\Delta^i(s) := \pi_{\theta_i}(s) - \pi_{\theta_{i+1}}(s).
\]
Then
\[
\eta \langle \nabla \ell_{s}(\pi_{\theta_i}), \pi_{\theta_i} - \pi_{\theta_{i+1}} \rangle
\le
\eta C \|\Delta^i(s)\|_1.
\]
Upon applying Young’s inequality $(ab\le \frac{a^2}{2}+\frac{b^2}{2})$, with $a = \eta C$ and $b=\Delta^i(s)$, we get
\[
\eta \langle \nabla \ell_{s}(\pi_{\theta_i}), \pi_{\theta_i} - \pi_{\theta_{i+1}} \rangle
\le
\frac{\eta^2 C^2}{2}
+
\frac{1}{2}\|\Delta^i(s)\|_1^2.
\]
Using Pinsker’s inequality,
\[
\frac{1}{2}\|\Delta^i(s)\|_1^2
\le
D_{\mathrm{KL}}(\pi_{\theta_{i+1}}(s)\|\pi_{\theta_i}(s)).
\]
Thus,
\begin{equation}
\eta \langle \nabla \ell_{s}(\pi_{\theta_i}), \pi_{\theta_i}(s) - \pi_{\theta_{i+1}}(s) \rangle
\le
\frac{\eta^2 C^2}{2}
+
D_{\mathrm{KL}}(\pi_{\theta_{i+1}}(s)\|\pi_{\theta_i}(s)).
\label{eqn:intermediate_kl_bound_new}
\end{equation}
Now, from convexity of \(\ell_s\), we have,
\[
\eta(\ell_s(\pi_{\theta_i})-\ell_s(\pi_T))
\le
\eta \langle \nabla \ell_s(\pi_{\theta_i}),\, \pi_{\theta_i}(s) - \pi_{\theta_{i+1}}(s) \rangle
+
\eta \langle \nabla \ell_s(\pi_{\theta_i}),\, \pi_{\theta_{i+1}}(s) - \pi_T(s) \rangle.
\]
Upon combining Equations \ref{eqn:three_point_new} and \ref{eqn:intermediate_kl_bound_new}, we obtain
\[
\eta(\ell_{s}(\pi_{\theta_i})-\ell_{s}(\pi_T))
\le
\frac{\eta^2 C^2}{2}
+
D_i(s) - D_{i+1}(s).
\]
Now, we reuse the definition of $D_i(s) = \ell_s(\pi_{\theta_i})-\ell_s(\pi_T)$, this gives
\[
\eta D_i(s)
\le
\frac{\eta^2 C^2}{2}
+
D_i(s) - D_{i+1}(s).
\]
Thus,
\begin{equation}
D_{i+1}(s)
\le
(1-\eta)D_i(s)
+
\frac{\eta^2 C^2}{2}.
\label{eqn:intermediate_recurssion}
\end{equation}
Next, we simplify this recursion. In particular, for a given state $s$, let us define
\[
a_i := D_i(s), \quad b := \frac{\eta^2 C^2}{2}.
\]
Therefore, Equation \ref{eqn:intermediate_recurssion} can be re-written as
\[
a_{i+1} \le (1-\eta)a_i + b.
\]

Using Lemma \ref{lem:recurrence_relation}, we get
\[
D_i(s) \le (1-\eta)^{i-1} D_1(s) + \frac{\eta^2 C^2}{2}\sum_{k=0}^{i-2}(1-\eta)^k, \quad \forall i \geq 1.
\]

Next, we bound the geometric sum:
\[
\sum_{k=0}^{i-2}(1-\eta)^k
\le
\sum_{k=0}^{\infty}(1-\eta)^k
=
\frac{1}{\eta}.
\]
Therefore,
\[
D_i(s)
\le
(1-\eta)^{i-1} D_1(s)
+
\frac{\eta^2 C^2}{2} \cdot \frac{1}{\eta}
=
(1-\eta)^{i-1} D_1(s)
+
\frac{\eta C^2}{2}.
\]
Using $D_1(s) \le C_0$, we obtain
\[
D_i(s)
\le
(1-\eta)^{i-1} C_0
+
\frac{\eta C^2}{2}.
\]
Now, considering Eq. \ref{eqn:loss_diff}, we have 
\[
\ell_i(\pi_{\theta_i})-\ell_i(\pi_T)
=
 \sum_{t=1}^H \mathbb{E}_{s_t \sim d_t^{\pi_{\theta_i}}}[D_i(s_t)]
\le
H(1-\eta)^{i-1}C_0 + H\frac{\eta C^2}{2}.
\]
Upon summing from $i=1$ to $n$ gives
\begin{align*}
\sum_{i=1}^{n}(\ell_i(\pi_{\theta_i})-\ell_i(\pi_T)) &\leq H\sum_{i=1}^{n} \left((1-\eta)^{i-1}C_0 + \frac{\eta C^2}{2}\right)\\
&=\frac{n\eta HC^2}{2} + HC_0\sum_{i=1}^{n}(1-\eta)^{i-1}\\
&\leq \frac{n\eta HC^2}{2} + \frac{HC_0}{\eta}\quad \quad \left(\because \sum_{i=1}^{n}(1-\eta)^{i-1} \leq \frac{1}{\eta}\right)
\end{align*} 
Upon dividing by $n$ gives
\begin{equation*}
    \frac{\sum_{i=1}^{n}(\ell_i(\pi_{\theta_i})-\ell_i(\pi_T))}{n} \leq \frac{\eta HC^2}{2} + \frac{HC_0}{\eta n}
\end{equation*}
Finally, choosing $
\eta = \sqrt{\frac{2C_0}{C^2 n}}$ completes the proof.

\appsubsection{Connections with Max RL}\label{subsec:app_max_rl_connection}

\subsubsection{Proof of Proposition \ref{prop:weighted_connection}: DistIL objective is a lower bound on the teacher-weighted likelihood of success}
\label{subsec:distil_max_rl_connection}

We begin with the definition of forward cross-entropy:
\[
H^{\times}\!\big(\pi_T(\cdot \mid x), \pi_\theta(\cdot \mid x)\big)
=
\sum_y \pi_T(y \mid x)\big[-\log \pi_\theta(y \mid x)\big].
\]
Next, we partition the sum into correct and incorrect responses:
\[
H^{\times}\!\big(\pi_T(\cdot \mid x), \pi_\theta(\cdot \mid x)\big)
=
\sum_{y: r(x,y)=1} \pi_T(y \mid x)\big[-\log \pi_\theta(y \mid x)\big]
+
\sum_{y: r(x,y)=0} \pi_T(y \mid x)\big[-\log \pi_\theta(y \mid x)\big].
\]

For any $y$ such that $r(x,y)=1$, we can write
\[
\pi_\theta(y \mid x)
=
p_{\pi_\theta}(x)\,\pi_\theta(y \mid x, r=1),
\]
where
\[
p_{\pi_\theta}(x)
=
\mathbb{P}_{y \sim \pi_\theta(\cdot \mid x)}[r(x,y)=1].
\]
Thus,
\[
-\log \pi_\theta(y \mid x)
=
-\log p_{\pi_\theta}(x)
-
\log \pi_\theta(y \mid x, r=1).
\]

Substituting this into the first term, we get
\begin{align*}
H^{\times}\!\big(\pi_T(\cdot \mid x), \pi_\theta(\cdot \mid x)\big)
=
- p_{\pi_T}(x)\log p_{\pi_\theta}(x)
&+
p_{\pi_T}(x)\,H^{\times}\!\big(\pi_T(\cdot \mid x,r=1), \pi_\theta(\cdot \mid x,r=1)\big)\\
&+
(1-p_{\pi_T}(x))\,\mathbb{E}_{y \sim \pi_T(\cdot \mid x,r=0)}[-\log \pi_\theta(y \mid x)].
\end{align*}

Next, we define
\[
\mathcal{R}(x)
:=
p_{\pi_T}(x)\,H^{\times}\!\big(\pi_T(\cdot \mid x,r=1), \pi_\theta(\cdot \mid x,r=1)\big)
+
(1-p_{\pi_T}(x))\,\mathbb{E}_{y \sim \pi_T(\cdot \mid x,r=0)}[-\log \pi_\theta(y \mid x)].
\]
Since both terms are nonnegative, we have $\mathcal{R}(x)\ge 0$. Therefore,
\[
H^{\times}\!\big(\pi_T(\cdot \mid x), \pi_\theta(\cdot \mid x)\big)
=
- p_{\pi_T}(x)\log p_{\pi_\theta}(x) + \mathcal{R}(x)
\ge
- p_{\pi_T}(x)\log p_{\pi_\theta}(x).
\]
Rearranging gives
\[
- H^{\times}\!\big(\pi_T(\cdot \mid x), \pi_\theta(\cdot \mid x)\big)
\le
p_{\pi_T}(x)\log p_{\pi_\theta}(x).
\]


\subsubsection{Reverse KL need not lower bound 
teacher-weighted likelihood of success}
\label{subsec:sdpo_max_rl_non_connection}

In Proposition~\ref{prop:weighted_connection}, we derived that DistIL objective is a lower bound on the teacher-weighted likelihood of success and therefore achieves better Pass@N. Next, we show that this need not hold true for SDPO. Proposition \ref{prop:reverse_kl_weighted_non_connection} formalizes this distinction by providing a counterexample in which the reverse-KL objective fails to provide such a lower bound. The construction considers a two-action setting in which the teacher assigns twice as much probability mass ($2\delta$) to the correct response as the student ($\delta$). As the teacher and student success probabilities simultaneously approach zero, the reverse-KL divergence scales as $O(\delta)$, whereas the magnitude of the teacher-weighted log-likelihood term,
$\bigl|p_{\pi_T}(x)\log p_{\pi_\theta}(x)\bigr|$, scales as $\Theta \bigl(\delta \log(1/\delta)\bigr)$. Since $\delta \log(1/\delta)$ dominates $\delta$ as $\delta \to 0$, the reverse-KL term becomes asymptotically negligible relative to the teacher-weighted log-likelihood term. Consequently, the reverse-KL objective need not provide a lower bound on the teacher-weighted log-likelihood of success.

\begin{proposition}[\textbf{Reverse KL need not lower bound 
teacher-weighted likelihood of success}]
\label{prop:reverse_kl_weighted_non_connection}
There exists a prompt $x$, a teacher policy $\pi_T$, and a student policy $\pi_\theta$ with $p_{\pi_T}(x) > p_{\pi_\theta}(x)$, i.e., the teacher is strictly better than the student, such that
\[
-D_{\mathrm{KL}}\!\big(\pi_\theta(\cdot\mid x)\,\|\,\pi_T(\cdot\mid x)\big)
\;>\;
p_{\pi_T}(x)\log p_{\pi_\theta}(x),
\]
and hence minimizing the reverse KL does not maximize a lower bound on the teacher-weighted log-likelihood of success.
\end{proposition}

Now, we provide a detailed analysis of Proposition \ref{prop:reverse_kl_weighted_non_connection}. Consider a two-response setting $\mathcal{Y}=\{y^+,y^-\}$ with $r(x,y^+)=1$ and $r(x,y^-)=0$. For $\delta\in\left(0,\tfrac{1}{2}\right)$, set
\[
\pi_T(y^+\mid x)=2\delta, \qquad \pi_\theta(y^+\mid x)=\delta,
\]
so that $p_{\pi_T}(x)=2\delta > \delta = p_{\pi_\theta}(x)$. The reverse KL divergence evaluates to
\[
D_{\mathrm{KL}}(\pi_\theta\|\pi_T)
=
\delta\log\frac{\delta}{2\delta}
+
(1-\delta)\log\frac{1-\delta}{1-2\delta}
=
-\delta\log 2
+
(1-\delta)\log\frac{1-\delta}{1-2\delta}
\;\approx\;
\delta(1-\log 2)
\quad\text{as }\delta\to 0,
\]
where we used the Taylor expansion $(1-\delta)\log\frac{1-\delta}{1-2\delta}\approx\delta$ 
as $\delta\to 0$. Meanwhile, the teacher-weighted log-likelihood of success satisfies
\[
p_{\pi_T}(x)\log p_{\pi_\theta}(x)
=
2\delta\log\delta,
\]
so $-p_{\pi_T}(x)\log p_{\pi_\theta}(x) = -2\delta\log\delta$. Since 
$-2\delta\log\delta \gg \delta(1-\log 2)$ as $\delta\to 0$, we conclude that for all sufficiently small $\delta>0$,
\[
D_{\mathrm{KL}}\!\big(\pi_\theta(\cdot\mid x)\,\|\,\pi_T(\cdot\mid x)\big)
\;<\;
-p_{\pi_T}(x)\log p_{\pi_\theta}(x),
\]
which is equivalent to
\[
-D_{\mathrm{KL}}\!\big(\pi_\theta(\cdot\mid x)\,\|\,\pi_T(\cdot\mid x)\big)
\;>\;
p_{\pi_T}(x)\log p_{\pi_\theta}(x).
\]
Thus, unlike the forward cross-entropy of DistIL, the negative reverse KL cannot serve as a lower bound on the teacher-weighted log-likelihood of success, even when the teacher is strictly better than the student.

\appsubsection{Technical Lemmas}

\begin{lemma}[\textbf{Fisher identity for softmax policies}]
\label{lem:fisher_identity}
Let $\pi_\theta$ be a softmax policy over a finite action set $\mathcal{A}$, with
\[
\pi_\theta(a) = \frac{e^{\theta_a}}{\sum_{b\in\mathcal{A}} e^{\theta_b}}.
\]
Define the score vector $s_a := \nabla_\theta \log \pi_\theta(a)$,
and the Fisher information matrix $F(\theta)
=
\mathbb{E}_{a\sim\pi_\theta}\!\left[s_a s_a^\top\right]$. Then, for any $a,b\in\mathcal{A}$,
\[
s_a^\top F(\theta)^{-1} s_b
=
\frac{\mathbf{1}\{a=b\}}{\pi_\theta(b)} - 1.
\]
Equivalently, $
\pi_\theta(a)\, s_a^\top F(\theta)^{-1} s_b
=
\mathbf{1}\{a=b\} - \pi_\theta(a)$.
\end{lemma}

\begin{proof}
Let $m:=|\mathcal{A}|$ and we write $p_a := \pi_\theta(a)$. Let $e_a\in\mathbb{R}^m$ denote the $a$-th standard basis vector, and let $\mathbf{1}\in\mathbb{R}^m$ denote the all-ones vector. Then, for the softmax parameterization, we have
\[
\frac{\partial \log \pi_\theta(a)}{\partial \theta_c}
=
\mathbf{1}\{a=c\} - p_c.
\]
Hence, $s_a = e_a - p$, where $p=(p_1,\dots,p_m)^\top$.

Now, by definition, $F(\theta)
=
\mathbb{E}_{a\sim p}\!\left[(e_a-p)(e_a-p)^\top\right]
$. Upon expanding, we get $F(\theta)
=
\mathbb{E}_{a\sim p}[e_a e_a^\top] - p p^\top$. Since $
\mathbb{E}_{a\sim p}[e_a e_a^\top] = \mathrm{diag}(p)$, we obtain
\[
F(\theta)=\mathrm{diag}(p)-pp^\top.
\]

Next, we fix $b\in\mathcal{A}$, and define $v := \frac{e_b}{p_b} - \mathbf{1}$. We claim that $F(\theta)v=s_b$. For this, we note that
$
p^\top v
=
\sum_{a\in\mathcal{A}} p_a\left(\frac{\mathbf{1}\{a=b\}}{p_b}-1\right)
=
1-1
=
0$.
Therefore,
\[
F(\theta)v
=
(\mathrm{diag}(p)-pp^\top)v
=
\mathrm{diag}(p)v - p(p^\top v)
=
\mathrm{diag}(p)v.
\]
Now,
\[
\mathrm{diag}(p)v
=
\mathrm{diag}(p)\left(\frac{e_b}{p_b}-\mathbf{1}\right)
=
e_b-p
=
s_b.
\]
Thus, \(F(\theta)v=s_b\). Now, assuming that inverse of $F(\theta)$ exists, we have
\[
F(\theta)^{-1} s_b = v = \frac{e_b}{p_b}-\mathbf{1}.
\]

Finally, using $(s_a=e_a-p)$, we have $
s_a^\top F(\theta)^{-1} s_b
=
s_a^\top\left(\frac{e_b}{p_b}-\mathbf{1}\right)$. Since \(s_a^\top \mathbf{1} = 0\), this simplifies to
\[
s_a^\top F(\theta)^{-1} s_b
=
\frac{s_a^\top e_b}{p_b}
=
\frac{\mathbf{1}\{a=b\}-p_b}{p_b}
=
\frac{\mathbf{1}\{a=b\}}{p_b} - 1.
\]
Moreover, multiplying both sides by $(p_a)$ gives the equivalent form
\[
p_a\, s_a^\top F(\theta)^{-1} s_b
=
\mathbf{1}\{a=b\} - p_a.
\]
\end{proof}

\begin{lemma}[\textbf{Trajectory score-function identity}]
\label{lem:trajectory_score_identity}
Let $\tau \sim p_\theta(\tau)$ be a differentiable trajectory distribution, and let $f(\tau,\theta)$ be any differentiable function. Then
\[
\nabla_\theta \mathbb{E}_{\tau \sim p_\theta}[f(\tau,\theta)]
=
\mathbb{E}_{\tau \sim p_\theta}
\big[
f(\tau,\theta)\nabla_\theta \log p_\theta(\tau)
+
\nabla_\theta f(\tau,\theta)
\big].
\]
In particular, if $f$ does not depend on $\theta$, then $
\nabla_\theta \mathbb{E}_{\tau \sim p_\theta}[f(\tau)]
=
\mathbb{E}_{\tau \sim p_\theta}
\big[
f(\tau)\nabla_\theta \log p_\theta(\tau)
\big]$. Moreover,
\[
\mathbb{E}_{\tau \sim p_\theta}\big[\nabla_\theta \log p_\theta(\tau)\big]=0.
\]
\end{lemma}

\begin{proof}
Using the definition of expectation and differentiating it gives,
\[
\nabla_\theta \mathbb{E}_{\tau \sim p_\theta}[f(\tau,\theta)]
=
\nabla_\theta \sum_\tau p_\theta(\tau) f(\tau,\theta)
=
\sum_\tau \nabla_\theta p_\theta(\tau)\,f(\tau,\theta)
+
\sum_\tau p_\theta(\tau)\nabla_\theta f(\tau,\theta).
\]
Next, we use $
\nabla_\theta p_\theta(\tau)=p_\theta(\tau)\nabla_\theta \log p_\theta(\tau)$, to get 
\begin{align*}
\nabla_\theta \mathbb{E}_{\tau \sim p_\theta}[f(\tau,\theta)] &=
\sum_\tau p_\theta(\tau) f(\tau,\theta)\nabla_\theta \log p_\theta(\tau)
+
\sum_\tau p_\theta(\tau)\nabla_\theta f(\tau,\theta)\\
&=\mathbb{E}_{\tau \sim p_\theta}
\big[
f(\tau,\theta)\nabla_\theta \log p_\theta(\tau)
+
\nabla_\theta f(\tau,\theta)
\big]
,
\end{align*}

Now, if \(f\) does not depend on \(\theta\), the second term vanishes. Finally,
\[
\mathbb{E}_{\tau \sim p_\theta}\big[\nabla_\theta \log p_\theta(\tau)\big]
=
\sum_\tau p_\theta(\tau)\nabla_\theta \log p_\theta(\tau)
=
\sum_\tau \nabla_\theta p_\theta(\tau)
=
\nabla_\theta \sum_\tau p_\theta(\tau)
=
\nabla_\theta 1
=
0.
\]
\end{proof}

\begin{lemma} Let \(\Phi:\Pi\to \mathbb{R}\) be a differentiable strictly convex function, and let
\[
D_\Phi(\pi\|\rho)
:=
\Phi(\pi)-\Phi(\rho)-\langle \nabla \Phi(\rho),\,\pi-\rho\rangle
\]
denote the Bregman divergence induced by \(\Phi\). For the following mirror-descent trust-region update:
\[
\pi^{i+1}\in \arg\min_{\pi\in\Pi}
\left\{
\langle \nabla \ell_i(\pi^i),\pi\rangle
+\frac{1}{\eta}D_\Phi(\pi\|\pi^i)
\right\},
\]
we have, for all \(\pi\in\Pi\),
\[
\big\langle \eta\nabla \ell_i(\pi^i)+\nabla \Phi(\pi^{i+1})-\nabla \Phi(\pi^i),\,\pi-\pi^{i+1}\big\rangle \ge 0.
\]
\label{lemma:mirror_descent_update}
\end{lemma}
\begin{proof}
The objective is convex in \(\pi\), and \(\pi^{i+1}\) is a minimizer over the convex set \(\Pi\). Therefore, the standard first-order optimality condition gives
\[
\big\langle \nabla \ell_i(\pi^i)+\frac{1}{\eta}\big(\nabla \Phi(\pi^{i+1})-\nabla \Phi(\pi^i)\big),\,\pi-\pi^{i+1}\big\rangle \ge 0,
\]
for all \(\pi\in\Pi\). Multiplying by \(\eta\) gives the final inequality.
\end{proof}

\begin{lemma} For any differentiable strictly convex \(\Phi\) and any \(u,v,w\) in its domain,
\[
\big\langle \nabla \Phi(u)-\nabla \Phi(v),\, w-u \big\rangle
=
D_\Phi(w,v)-D_\Phi(w,u)-D_\Phi(u,v).
\]
\label{lemma:bregman_three_point}
\end{lemma}
\begin{proof}
By definition,
\[
D_\Phi(x,y)=\Phi(x)-\Phi(y)-\langle \nabla \Phi(y),x-y\rangle.
\]
Upon expanding
\(D_\Phi(w,v)-D_\Phi(w,u)-D_\Phi(u,v)\)
and rearranging terms gives
\[
\Phi(w)-\Phi(v)-\langle \nabla\Phi(v),w-v\rangle
-\Phi(w)+\Phi(u)+\langle \nabla\Phi(u),w-u\rangle
-\Phi(u)+\Phi(v)+\langle \nabla\Phi(v),u-v\rangle,
\]
which simplifies to
\[
\langle \nabla \Phi(u)-\nabla \Phi(v),\, w-u \rangle.
\]
\end{proof}

\begin{lemma}
Let $\{a_i\}_{i\ge1}$ be a sequence satisfying $a_{i+1} \le (1-\eta)a_i + b$, where $\eta \in (0,1)$. Then, for all $i \ge 1$, 
\[
a_i \le (1-\eta)^{i-1} a_1 + b \sum_{k=0}^{i-2}(1-\eta)^k.
\]
\label{lem:recurrence_relation}
\end{lemma}

\begin{proof}
We prove by induction on $i$. For $i=1$, the bound reduces to $a_1 \le a_1$, which holds trivially. Next, assume the statement holds for some $i \ge 1$. Using the recursion and the induction hypothesis,
\begin{align*}
a_{i+1}
&\le
(1-\eta)\left[(1-\eta)^{i-1} a_1 + b \sum_{k=0}^{i-2}(1-\eta)^k\right] + b\\
&= (1-\eta)^i a_1
+
b \sum_{k=0}^{i-2}(1-\eta)^{k+1}
+
b\\
&=(1-\eta)^i a_1
+  b \sum_{k=0}^{i-1}(1-\eta)^k.
\end{align*}
This establishes the claim for $i+1$ and completes the proof.
\end{proof}

\appsection{Experiment Details}
\label{app:exp_detail}

We use $4$ H200 GPUs for all our experiments. For prompts and implementation, we use SDPO's codebase \citep{sdpo_github} on Science and Coding benchmarks whereas OPSD's codebase \citep{zhao2026opsd} on Mathematics benchmarks.

For DistIL, we observe performance degradation for longer sequences $y \sim \pi_{\theta}(\cdot \mid x)$, as the cumulative future loss $\sum_{i>t}\mathrm{H}^{\times}(\cdot)$ grows with sequence length, leading to disproportionately large gradient magnitudes. To mitigate this effect, we apply length normalization, similar to \citet{gu2024minillm}. Specifically, we replace the cumulative loss in future-credit assignment in Eq. \eqref{eqn:grad_expression} with
\[
\tilde{L}_{t+1}
=
\frac{1}{T - t - 1}
\sum_{i > t}
H^{\times}\big(
\text{stopgrad}\big(\pi_{\theta}(\cdot \mid x, y_{1:i-1}, f)\big)
\;\|\;
\pi_\theta(\cdot \mid s_i)
\big).
\]

\appsubsection{Hyperparameters for Science and Coding Benchmark}
For Science and Coding benchmarks, we follow the implementation of SDPO and use their reported hyperparameters for baselines. The hyperparameters for SDPO, DistIL (Ours), and GRPO with its off-policy and on-policy variants are provided in Table \ref{tab:sdpo_hyperparameters_sci_coding}, \ref{tab:distil_hyperparameters_sci_coding}, and \ref{tab:grpo_hyperparameters_sci_coding} respectively.

\appsubsection{Hyperparameters for Maths task}
\label{app:eval_hyper_math}
For Mathematics benchmarks, we follow the implementation of OPSD and use their reported hyperparameters for baselines. The hyperparameters for SDPO, OPSD, DistIL (Ours), and GRPO are provided in Table \ref{tab:maths_train_hyperparameters} and that for SFT in Table \ref{tab:sft_hyperparameters}. Evaluation hyperparameters are provided in Table \ref{tab:eval_hyperparameters_qwen4b} for Qwen3-4B-2507-Instruct model and Table \ref{tab:eval_hyperparameters_qwen8b} for Qwen3-8B model. The only difference between OPSD and SDPO in this scenario is that OPSD uses Forward-KL divergence whereas SDPO uses reverse-KL divergence in these mathematics experiments.

\begin{figure}[t]
    \centering
    \includegraphics[width=0.24\linewidth]{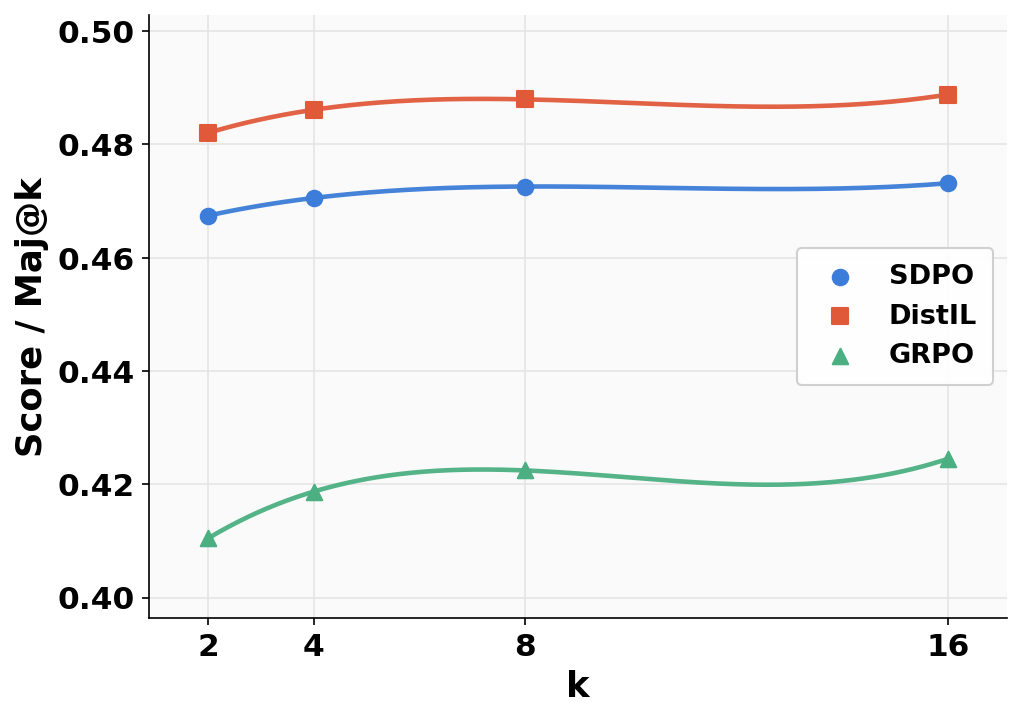}
    \includegraphics[width=0.24\linewidth]{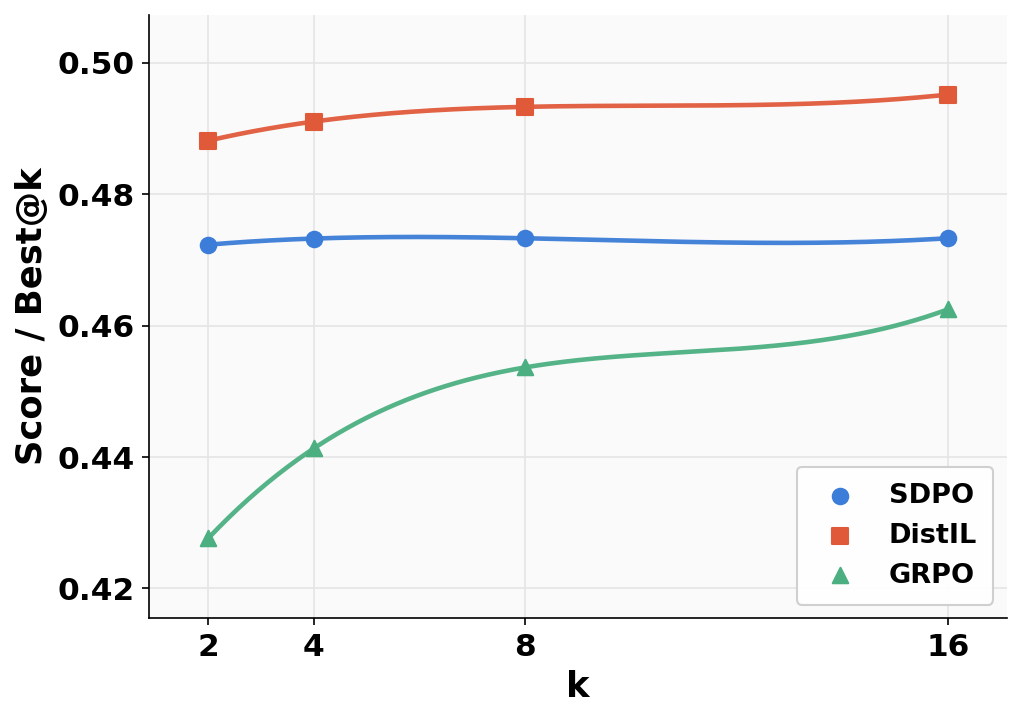}
    \includegraphics[width=0.24\linewidth]{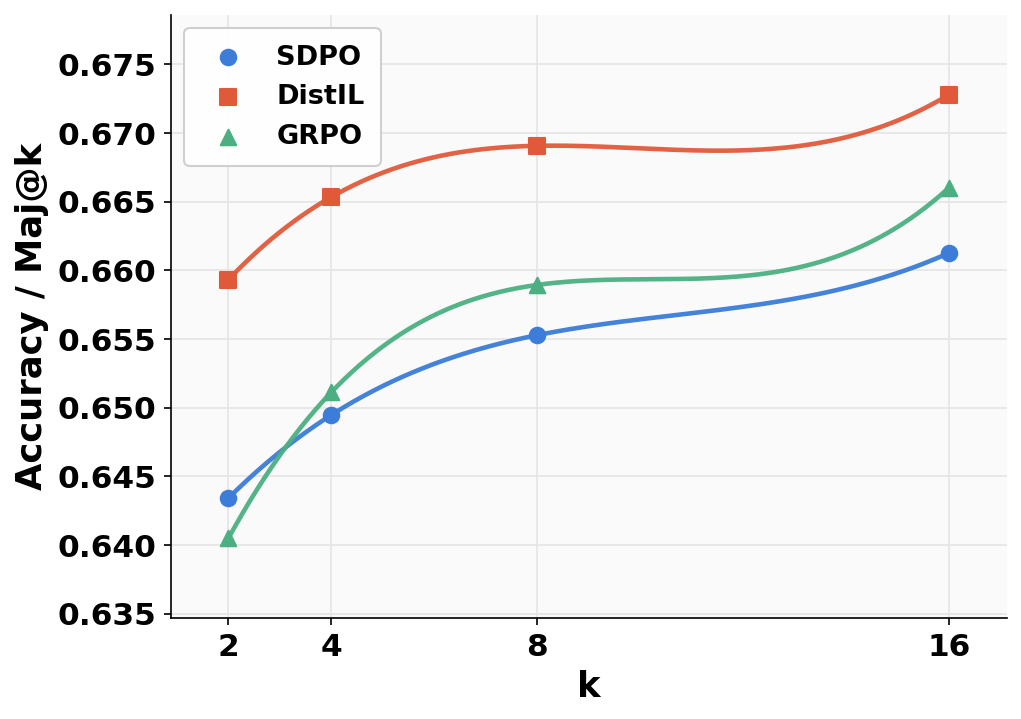}
    \includegraphics[width=0.24\linewidth]{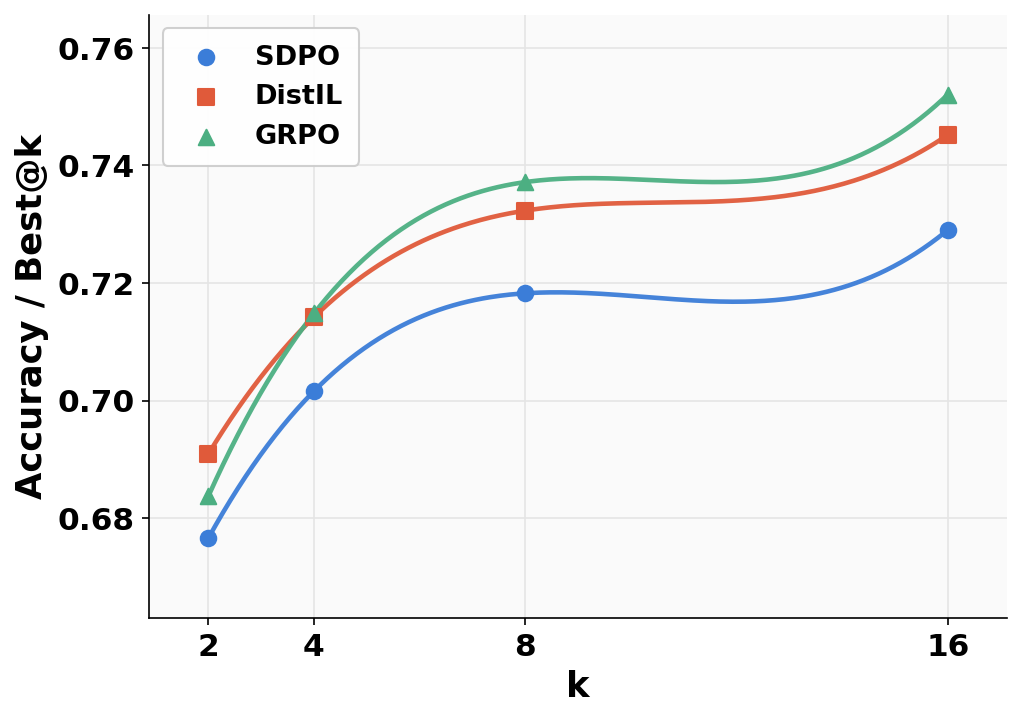}
    \caption{LCBv6 evaluation at $\tau{=}0.2$ (checkpoint-step 80), reporting Score and Accuracy at Best@$k$ and Maj@$k$ for $k \in \{2,4,8,16\}$.}
    \label{fig:lcb_eval_temp_0.6}
\end{figure}

\appsection{Additional Results}

\appsubsection{More results for Coding.}
SDPO evaluated their coding experiment at temperature = 0.6. We include results for that scenario in here. DistIL achieves Acc/Mean@16$\,{=}\,0.660$ and Score/Mean@16$\,{=}\,0.482$, outperforming
SDPO ($0.643$, $0.467$) and GRPO ($0.639$, $0.411$). The large gap over GRPO reflects its
fundamental inability to exploit execution feedback. Figure~\ref{fig:lcb_eval_temp_0.6}  notice that the same trend follows as that for temperature=$0.2$, except that GRPO becomes slightly better on Accuracy/Best@k for k=16.

\begin{table}[t]
\centering
\caption{Hyperparameters used for \textbf{SDPO} for each experimental setup.}
\small
\begin{tabular}{@{}l cc@{}}
\toprule
\textbf{Parameters} & \textbf{Science} & \textbf{Coding} \\
\midrule
\multicolumn{3}{@{}l}{\textbf{General}} \\
Model & \makecell[l]{Qwen/Qwen3-8B \\ allenai/Olmo3-7B-Instruct} & Qwen/Qwen3-8B \\
Thinking & False & False \\
\midrule
\multicolumn{3}{@{}l}{\textbf{Data}} \\
Max.\ prompt length   & 2048 & 2048 \\
Max.\ response length & 8192 & 8192 \\
\midrule
\multicolumn{3}{@{}l}{\textbf{Batching}} \\
Question batch size & 32 & 32 \\
Mini batch size     & 32 & 1  \\
Number of rollouts  & 8  & 8  \\
\midrule
\multicolumn{3}{@{}l}{\textbf{Rollout}} \\
Inference engine & vllm & vllm \\
Temperature      & 1.0  & 1.0  \\
\midrule
\multicolumn{3}{@{}l}{\textbf{Validation}} \\
Number of rollouts & 16   & 4    \\
Temperature        & 0.6  & 0.6  \\
Top-$p$            & 0.95 & 0.95 \\
\midrule
\multicolumn{3}{@{}l}{\textbf{SDPO Loss}} \\
Top-$K$ distillation             & 100             & 20         \\
Distillation divergence          & Jensen--Shannon & Reverse-KL \\
Clip advantages                  & --              & --         \\
Teacher-EMA update rate          & 0.05            & 0.01       \\
Rollout importance sampling clip & 2               & 2          \\
\midrule
\multicolumn{3}{@{}l}{\textbf{Training}} \\
Optimizer          & AdamW & AdamW \\
Learning rate      & $1\times10^{-5}$ (constant) & $1\times10^{-6}$ (constant) \\
Warmup steps       & 10   & 0    \\
Weight decay       & 0.01 & 0.01 \\
Gradient clip norm & 1.0  & 1.0  \\
\bottomrule
\end{tabular}
\label{tab:sdpo_hyperparameters_sci_coding}
\end{table}

\begin{table}[t]
\centering
\caption{Hyperparameters used for \textbf{DistIL} for each experimental setup.}
\small
\begin{tabular}{@{}l cc@{}}
\toprule
\textbf{Parameters} & \textbf{Science} & \textbf{Coding} \\
\midrule
\multicolumn{3}{@{}l}{\textbf{General}} \\
Model & \makecell[l]{Qwen/Qwen3-8B \\ allenai/Olmo3-7B-Instruct} & Qwen/Qwen3-8B \\
Thinking & False & False \\
\midrule
\multicolumn{3}{@{}l}{\textbf{Data}} \\
Max.\ prompt length   & 2048 & 2048 \\
Max.\ response length & 8192 & 8192 \\
\midrule
\multicolumn{3}{@{}l}{\textbf{Batching}} \\
Question batch size & 32 & 32 \\
Mini batch size     & 32 & 1  \\
Number of rollouts  & 8  & 8  \\
\midrule
\multicolumn{3}{@{}l}{\textbf{Rollout}} \\
Inference engine & vllm & vllm \\
Temperature      & 1.0  & 1.0  \\
\midrule
\multicolumn{3}{@{}l}{\textbf{Validation}} \\
Number of rollouts & 16   & 4    \\
Temperature        & 0.6  & 0.6  \\
Top-$p$            & 0.95 & 0.95 \\
\midrule
\multicolumn{3}{@{}l}{\textbf{SDPO Loss}} \\
Top-$K$ distillation             & 100             & 20         \\
Distillation divergence          & Jensen--Shannon & Reverse-KL \\
Clip advantages                  & --              & --         \\
Teacher-EMA update rate          & 0.01            & 0.01       \\
Rollout importance sampling clip & 2               & 2          \\
\midrule
\multicolumn{3}{@{}l}{\textbf{Training}} \\
Optimizer          & AdamW & AdamW \\
Learning rate      & $5\times10^{-5}$ (constant) & $1\times10^{-6}$ (constant) \\
Warmup steps       & 10   & 0    \\
Weight decay       & 0.01 & 0.01 \\
Gradient clip norm & 1.0  & 1.0  \\
\bottomrule
\end{tabular}
\label{tab:distil_hyperparameters_sci_coding}
\end{table}

\begin{table}[t]
\centering
\caption{Hyperparameters used for \textbf{GRPO} (On-Policy) and (Off-Policy) for Science and Coding.}
\small
\begin{tabular}{@{}l l@{}}
\toprule
\textbf{Parameters} & \textbf{Value} \\
\midrule
\multicolumn{2}{@{}l}{\textbf{General}} \\
Model   & \makecell[l]{Qwen/Qwen3-8B \\ allenai/Olmo3-7B-Instruct} \\
Thinking & False \\
\midrule
\multicolumn{2}{@{}l}{\textbf{Data}} \\
Max.\ prompt length   & 2048 \\
Max.\ response length & 8192 \\
\midrule
\multicolumn{2}{@{}l}{\textbf{Batching}} \\
Question batch size & 32 \\
Mini batch size     & 8 (default) / 32 (on-policy) \\
Number of rollouts  & 8 \\
\midrule
\multicolumn{2}{@{}l}{\textbf{Rollout}} \\
Inference engine & vllm \\
Temperature      & 1.0 \\
\midrule
\multicolumn{2}{@{}l}{\textbf{Validation}} \\
Temperature        & 0.6 \\
Top-$p$            & 0.95 \\
Number of rollouts & 16 \\
\midrule
\multicolumn{2}{@{}l}{\textbf{Loss}} \\
$\epsilon$-high                  & 0.28 \\
Rollout importance sampling clip & 2 \\
KL coefficient ($\lambda$)       & 0.0 \\
\midrule
\multicolumn{2}{@{}l}{\textbf{Training}} \\
Optimizer          & AdamW \\
Learning rate      & $1\times10^{-6}$ (default) / $1\times10^{-5}$ (on-policy) \\
Warmup steps       & 10 \\
Weight decay       & 0.01 \\
Gradient clip norm & 1.0 \\
\bottomrule
\end{tabular}
\label{tab:grpo_hyperparameters_sci_coding}
\end{table}

\begin{table}[t]
\centering
\caption{Training configuration for GRPO, OPSD, SDPO, DistIL on Mathematics benchmarks.}
\small
\begin{tabular}{@{}l cc@{}}
\toprule
\textbf{Parameter} & \textbf{GRPO} & \textbf{OPSD/SDPO/DistIL} \\
\midrule
Learning Rate         & $5\times10^{-6}$ & $5\times10^{-6}$ \\
Effective Batch Size  & 32 & 32 \\
\midrule
LoRA Rank ($r$)       & 64 & 64 \\
LoRA Alpha ($\alpha$) & 128 & 128 \\
LoRA Target Modules   & \multicolumn{2}{c}{\texttt{q\_proj, k\_proj, v\_proj, o\_proj, gate\_proj, up\_proj, down\_proj}} \\
\midrule
Max Completion Length & 16{,}000 & 16{,}384 \\
\midrule
Number of Generations per Prompt & 8   & 1   \\
Sampling Temperature             & 0.7 & 0.7 \\
KL Coefficient ($\beta$)         & 0.0 & --  \\
Training Steps                   & 500 & 100 \\
\bottomrule
\end{tabular}
\label{tab:maths_train_hyperparameters}
\end{table}

\begin{table}[t]
\centering
\caption{Training configuration for SFT on Mathematics benchmarks.}
\small
\begin{tabular}{@{}l l@{}}
\toprule
\textbf{Parameter} & \textbf{SFT} \\
\midrule
Learning Rate         & $5\times10^{-6}$ \\
Effective Batch Size  & 32 \\
\midrule
LoRA Rank ($r$)       & 64 \\
LoRA Alpha ($\alpha$) & 128 \\
LoRA Target Modules   & \texttt{q\_proj, k\_proj, v\_proj, o\_proj, gate\_proj, up\_proj, down\_proj} \\
\midrule
Max Sequence Length      & 16{,}000 \\
Number of Training Steps & 100 \\
\bottomrule
\end{tabular}
\label{tab:sft_hyperparameters}
\end{table}

\begin{table}[t]
\centering
\caption{Evaluation parameters for Qwen3-4B-Instruct-2507 Model on Maths Task.}
\small
\begin{tabular}{@{}l l@{}}
\toprule
\textbf{Parameter} & \textbf{Value} \\
\midrule
Max New Tokens       & 16{,}384 \\
Thinking Mode        & Enabled \\
Top-$p$              & 0.95 \\
Top-$k$              & $20$ \\
Min-$p$              & 0.0 \\
Presence Penalty     & 0.0 \\
Samples per Prompt   & 64 \\
Temperature          & 0.7 \\
\bottomrule
\end{tabular}
\label{tab:eval_hyperparameters_qwen4b}
\end{table}

\begin{table}[t]
\centering
\caption{Evaluation parameters for Qwen3-8B Model on Maths Task.}
\small
\begin{tabular}{@{}l l@{}}
\toprule
\textbf{Parameter} & \textbf{Value} \\
\midrule
Max New Tokens       & 38{,}912 \\
Thinking Mode        & Enabled \\
Top-$p$              & 1.0 \\
Top-$k$              & $-1$ \\
Min-$p$              & 0.0 \\
Presence Penalty     & 0.0 \\
Samples per Prompt   & 64 \\
Temperature          & 1.0 \\
\bottomrule
\end{tabular}
\label{tab:eval_hyperparameters_qwen8b}
\end{table}

\end{document}